\documentclass[twoside]{article}

\usepackage{amsmath, amssymb, amsthm, mathtools}
\usepackage{graphicx, subcaption}
\newtheorem{theorem}{Theorem}
\newtheorem{proposition}{Proposition}
\newtheorem{lemma}{Lemma}
\newtheorem{corollary}{Corollary}
\theoremstyle{definition}
\newtheorem{definition}{Definition}

\newtheorem{remark}{Remark}

\usepackage[preprint]{aistats2027}
\usepackage[round]{natbib}

\begin{document}

%
\runningtitle{Identifiability from Stationary Covariance}

%

\twocolumn[

\aistatstitle{Identifiability from Stationary Covariance in Linear SDEs with Reciprocal Interactions and Latent Variables}

\aistatsauthor{Ana Letícia Garcez Vicente \And Saber Salehkaleybar}

\aistatsaddress{Leiden Institute of Advanced Computer Science (LIACS), Leiden University, Netherlands} ]

\begin{abstract}
  Continuous-time stochastic differential equations (SDEs) provide a natural framework for modeling interacting dynamical systems. A central causal question is whether causal effects among observed variables can be identified from observational data drawn from the stationary distribution, given the underlying causal graph. For linear SDEs, existing identifiability results are largely restricted to fully observed systems whose causal graphs contain no directed 2-cycles. We first establish new graphical conditions for the generic identifiability of causal effects in systems that may contain directed 2-cycles, thereby allowing for reciprocal interactions. Next, we address the previously unexplored setting of latent variables by developing an instrumental-variable framework for identifying a target direct causal effect under partial observability.
\end{abstract}

\section{Introduction}

Understanding causal interactions in dynamical systems is a fundamental problem across biology \citep{marbach2012wisdom, lorch2026latent}, economics \citep{hamilton2020time}, and engineering \citep{aastrom2012introduction}. Stochastic differential equations (SDEs), in which the diffusion captures external stochastic perturbations and the drift function describes the deterministic interactions among variables, are a natural model for many such systems that evolve continuously over time \citep{oksendal2003stochastic}. Recovering the drift coefficients shows the local direct causal effects controlling the system dynamics. However, temporal trajectories are either not available or only partially observed in many applications. Instead, data consist of samples collected from a system operating near equilibrium. In this context, stationary SDEs have emerged as an attractive framework for causal modeling \citep{varando2020graphical, lorch2024causal}. 

Under standard stability assumptions, stationary linear SDEs satisfy the continuous Lyapunov equation, which provides an algebraic relationship between the stationary covariance matrix and the unknown drift matrix. This connection enables causal information to be inferred from equilibrium observations without requiring access to the underlying trajectories. Furthermore, graphical continuous Lyapunov models provide a causal graphical representation analogous to structural causal models while naturally accommodating cycles and self-loops, common features 
in real dynamical systems \citep{varando2020graphical, hansen2025trek, fitch2019learning}.

In linear SDEs, a fundamental challenge is whether the drift matrix can be recovered from the stationary covariance matrix when its sparsity pattern is known. Recent work has shown that the drift matrix is globally identifiable when the underlying graph is simple, i.e., when it contains no directed cycles of length two \citep{dettling2023identifiability}. Reciprocal interactions, therefore, constitute the principal obstruction to identifiability. Directed 2-cycles naturally arise in many applications. In biological regulatory networks, for example, mutually activating and mutually repressing components form fundamental feedback motifs responsible for bistability and switch-like behavior \citep{ferrell2002self}. Reciprocal interactions between cortical and subcortical are also ubiquitous in neural circuits \citep{doi:10.1152/physrev.00048.2021}. Similar bidirectional interactions are also present in engineered systems with coupled controllers and feedback loops \citep{aastrom2021feedback}. Whether directed 2-cycles necessarily prevent identification, or whether additional graphical structure can resolve the resulting ambiguity, remains largely unexplored.

Another challenge arises when the system is only partially observed. In many applications, latent variables directly influence the variables of interest, introducing confounding that prevents direct causal effects from being recovered through standard adjustment techniques. Although latent confounding has been extensively studied in structural causal models through instrumental variables, adjustment criteria \cite{perkovic2018complete}, and do-calculus \cite{pearl2009causality}, comparable findings for stationary SDEs are far less established.

In this paper, we study identification from stationary covariance in non-simple linear SDEs through a common rank-based perspective. Directed 2-cycles are allowed throughout. 
Under full observability, we study graphical conditions that guarantee identification of the complete drift matrix through a global rank condition derived from the stationary Lyapunov equation. Under partial observability, latent variables prevent full drift recovery, but suitable instruments lead to a target-row rank condition for identifying coefficients of interest. Both settings follow the same principle: using graphical structure to certify rank conditions for covariance-based 
systems induced by the stationary Lyapunov equation. Our main contributions are:
\begin{itemize}
    \item We establish new graphical conditions for generic identifiability of graphical continuous Lyapunov models containing directed $2$-cycles (Theorem \ref{thm:3v_motif}). We further characterize several identifiable four-node motifs that fall outside these conditions (Proposition \ref{prop:one-vertex-extensions}) and show how identifiability can be localized to identifiable subgraphs within larger non-identifiable systems (Theorem \ref{thm:autonomous-subgraph});
    \item We provide an instrumental-variable framework for identifying direct drift coefficients under latent confounding using only observed covariance matrix (Theorem \ref{thm:ivthm3});
    \item We provide graphical conditions guaranteeing the full-rank requirement of the instrumental covariance matrix. We completely characterize the direct-instrument case via a bipartite matching condition (Theorem \ref{thm:directIV}) and derive a sufficient graphical criterion based on vertex-disjoint directed paths when intermediate variables are present (Theorem \ref{thm:nondirectIV}).
\end{itemize}

\section{Related Work}
\paragraph{Causal modeling with stochastic differential equations.}
Classical causal inference is predominantly formulated through structural causal models \citep{pearl2009causality}. Extending causal reasoning to dynamical systems requires accounting for temporal dynamics and feedback. \cite{hansen2014causal} formalized interventions on stochastic differential equations, while \cite{bongers2018random,mogensen2018causal} developed causal models for general stochastic dynamical systems.

When temporal trajectories are unavailable, inference may rely on stationary observations. Initial work investigated learning directed graphical models from stationary Gaussian observations \citep{fitch2019learning}, while \cite{varando2020graphical} introduced graphical continuous Lyapunov models, in which the sparsity pattern of the drift matrix defines a directed graph. In this representation, a directed edge corresponds to a direct interaction in the drift, while the stationary covariance summarizes the equilibrium dependence induced by the entire dynamical system. Subsequent work established a trek rule providing a graphical characterization of stationary covariance entries \citep{hansen2025trek}, while \cite{lorch2024causal, lorch2026latent} developed methods for causal learning directly from stationary distributions.

A fundamental question is when the drift matrix is identifiable from stationary covariance information. \cite{dettling2023identifiability} proved that graphical continuous Lyapunov models are globally identifiable if and only if the underlying graph is simple, and further established necessary conditions for generic identifiability. Our work complements this line of research by identifying graphical conditions under which non-simple graphs remain generically identifiable and by characterizing local graphical structures that repair the ambiguity introduced by directed reciprocal interactions.

\paragraph{Causal effects under partial observation}
Partial observation is a central challenge in causal inference where unobserved common causes can invalidate adjustment, induce spurious associations, and prevent causal effects from being recovered from the observed distribution alone. In structural causal models, valid adjustment set or do-calculus provide solutions to  identify causal effects in the presence of latent variables \citep{pearl2009causality,richardson2002ancestral,shpitser2006identification,shpitser2012validity,perkovic2018complete}. However, these approaches are mainly developed for acyclic structural causal models where no directed cycle is allowed.

In linear structural equation models, generic identifiability from observational covariance has been studied through graphical rank criteria, including the half-trek criterion \citep{foygel2022half}, instrumental cutsets \citep{kumor2019efficient}, and latent-variable extensions \citep{foygel2012half}. These approaches, like ours, use graphical path structures to certify rank conditions. Our setting differs in that identification is based on stationary covariance constrained by the continuous Lyapunov equation.

For stationary SDEs, considerably less is known. Existing work primarily focused on recovering the complete drift matrix under full observation  \citep{dettling2023identifiability}. In many applications, however, recovering the entire dynamical system is unnecessary; instead, the objective is to identify a particular direct causal effect despite latent confounding. More recently, \cite{van2026sign} investigated sign identifiability from stationary second-order information, including settings with latent variables. Additionally, instrumental processes whose identifying moment equations are built from integrated covariance information have been used to extend instrumental-variable approaches to stochastic processes observed over time \citep{mogensen2023instrumental}. Our second contribution differs from these approaches. Rather than recovering the entire drift matrix, identifying only the sign of causal effects, or relying on trajectory observations to estimate integrated covariances, we study identification of a target drift coefficient under latent confounding. We derive instrumental equations and establish graphical conditions under which the observed stationary covariance uniquely identifies the target causal effects.

\section{Preliminaries}
\subsection{Stochastic differential equations}
Stochastic differential equations (SDEs) provide a natural framework for modeling continuous-time systems evolving both deterministic interactions and random perturbations indexed by node set $V$, with $|V| = d$. Let $\textbf{X}(t) \in \mathbb{R}^d$ denote the state vector at time $t$. Each component $X_v(t)$ represents variable $v \in V$ whose evolution is influenced by other variables in the system as well as by external stochastic perturbations. Our analysis focuses on the linear SDEs:
\begin{equation}
    d\textbf{X}(t)=\textbf{M}\textbf{X}(t)dt+\textbf{D}d\textbf{W}(t),
\end{equation}
which corresponds to the multivariate Ornstein-Uhlenbeck (OU) process. $\textbf{D}$ is the diffusion matrix and $\textbf{W}(t)$ is a $d$-dimensional standard Brownian motion. The matrix $\textbf{M} =(m_{ij})_{i, j \in V}\in \mathbb{R}^{d\times d}$ is the drift matrix and characterizes the deterministic interactions among variables. In particular, the entry $m_{ij}$ quantifies the direct effect of variable $X_j$ on the evolution of $X_i$. Therefore, from a causal perspective, $m_{ij}$ represents the direct causal effect from $X_j$ to $X_i$.

Let $\textbf{C}:= \textbf{DD}^\top$. Under the stability condition that all eigenvalues of $\textbf{M}$ have strictly negative real part, the process admits an unique stationary Gaussian distribution with covariance matrix $\boldsymbol{\Sigma}$ satisfying:
\begin{equation} \label{eq:lyapunov_stat}
    \textbf{M}\boldsymbol{\Sigma}+\boldsymbol{\Sigma} \textbf{M}^\top + \textbf{C} = 0.
\end{equation}
Following \cite{dettling2023identifiability}, we assume that the diffusion matrix $\textbf{C}$ is known.

\subsubsection{Graphical Representation}
The drift matrix can be associated with a directed graph $G = (V, E)$, called the drift graph. Each node in $V$ corresponds to one component of the stochastic process $\textbf{X} = (X_1, \cdots, X_d)$, and a directed edge $X_i \to X_j$ is present whenever the drift coefficient describing the influence of $X_i$ on $X_j$ is nonzero, i.e., $m_{ji} \neq 0$. The variables are partitioned into $V = O \cup L$, where $O$ denotes the observed variables, and $L$ are the unobserved/latent ones.

For a node $x\in V$, we denote the set of parents of $x$ $\operatorname{Pa}(x) = \{v\in V:v\to x\}$ the set of nodes with a direct edge pointing to $x$. Similarly, the set of children of $x$ is $\operatorname{Ch}(x) = \{v\in V:x\to v\}$. A directed path from a node $x$ to $y$ ($x \rightsquigarrow y$) is a sequence of distinct nodes $x=v_0\rightarrow\cdots\rightarrow v_k=y$, such that every consecutive pair of nodes is connected by a directed edge oriented in the same direction. The ancestors of $x$, denoted by $\operatorname{An}(x)$, are the nodes from which there exists a directed path to $x$, including $x$ itself. Conversely, the descendants, $\operatorname{De}(x)$, are the nodes reachable from $x$ through directed paths, including $x$.

The graph $G$ may contain directed cycles. A directed cycle of length $k$ is a sequence of distinct nodes $v_1 \to \cdots \to v_k \to v_1$, with $k \geq 2$. In particular, a cycle of length two, referenced as \textit{a $2$-cycle} throughout this paper, consists of two reciprocal edges $x \to y$ and $y \to x$ representing mutual feedback. A graph without directed cycles is called a Directed Acyclic Graph (DAG). Our framework does not assume acyclicity.

Under the stability assumption, the stationary distribution of the linear SDE driven by a Wiener process is multivariate Gaussian. Without loss of generality, we assume that the stationary mean is zero. Consequently, the stationary distribution is completely characterized by its covariance matrix, and our identifiability analysis can therefore be formulated entirely in terms of stationary covariances. 

A trek between two nodes $x$ and $y$ is a path without colliders, i.e., a path that can be decomposed into two directed paths sharing a common source node $c$, $x\leftarrow\cdots\leftarrow c \rightarrow \cdots \rightarrow y.$ The node $c$ is called the top of the trek. A trek may degenerate into a purely directed path if $c = x$ or $c = y$. Treks play a central role in OU models with Wiener process. Each entry $(i,j)$ of stationary covariance admits a decomposition as sums of contributions over treks connecting the pair $(i,j)$. Therefore, the existence or absence of treks determines whether two variables may exhibit nonzero covariance. In particular, if no trek connects two nodes $X_i$ and $X_j$, then $\Sigma_{ij} = 0$, implying that $X_i$ and $X_j$ are marginally independent \citep{boege2025conditional}.

\subsubsection{Problem Formulation}
Given a directed graph $G = (V, E)$ and a known positive definite diagonal diffusion matrix $\textbf{C}$, we consider the graphical continuous Lyapunov model $\mathcal{M}_{G, \textbf{C}} = \{\boldsymbol{\Sigma} : \textbf{M}\boldsymbol{\Sigma} + \boldsymbol{\Sigma}\textbf{M}^\top + \textbf{C} = 0, supp(\textbf{M})\subseteq E, \textbf{M} \textit{ stable}\}$ where $\operatorname{supp}(\textbf{M}) = \{(i,j):m_{ij} \neq 0\}$. The paper studies two identification problems over $\mathcal{M}_{G, \mathbf{C}}$: (i) recovery of the entire drift matrix under full observability, and (ii) recovery of a target direct causal effect $m_{yx}$ under partial observability. 
\begin{definition}[Global identifiability]
    A parameter $\theta (\textbf{M})$ is globally identifiable in $\mathcal{M}_{G, \textbf{C}}$ if for any two admissible drift matrices $\boldsymbol{\Sigma}(\textbf{M}_1) = \boldsymbol{\Sigma}(\textbf{M}_2) \implies \theta(\textbf{M}_1) = \theta(\textbf{M}_2)$.
\end{definition}
\begin{definition}[Generic identifiability]
    A parameter $\theta (\textbf{M})$ is generically identifiable in $\mathcal{M}_{G, \textbf{C}}$ if there exists a proper algebraic subset $\mathcal{A}$ of the admissible parameter space such that $\boldsymbol{\Sigma}(\textbf{M}_1) = \boldsymbol{\Sigma}(\textbf{M}_2) \implies \theta(\textbf{M}_1) = \theta(\textbf{M}_2)$ for all admissible $\textbf{M}_1, \textbf{M}_2 \notin \mathcal{A}$.
\end{definition}
\paragraph{Kernel matrix and rank criterion} Our analysis of full observability relies on the algebraic characterization of identifiability \cite{dettling2023identifiability}. For $\boldsymbol{\Sigma}\in\mathcal{M}_{G,\mathbf{C}}$, define:
\[
\textbf{H}(\boldsymbol{\Sigma})_{i\rightarrow j,(k,l)}
=
\begin{cases}
-\Sigma_{lj}, & i=k,\\[1mm]
\Sigma_{kj}, & i=l,\\[1mm]
0, & \text{otherwise},
\end{cases}
\]
where the rows are indexed by all ordered pairs $(i,j)\in V \times V$ (corresponding to all possible directed edges, including self-loops), and the columns are indexed by unordered pairs $(k,l)$ with $k<l$. For a graph $G=(V,E)$, let $\textbf{H}(\boldsymbol{\Sigma})_{E^c,\cdot}$ denote the submatrix obtained by restricting the rows of $\textbf{H}(\boldsymbol{\Sigma})$ to those corresponding to the non-edges $E^c$. Throughout the paper, we certify generic identifiability by exhibiting square submatrices (minors) of $\textbf{H}(\boldsymbol{\Sigma})_{E^c,\cdot}$ whose determinants are not identically zero.

The following characterization is the main algebraic criterion used throughout the paper.
\begin{lemma}[{\cite[Lemma~6.4(ii)]{dettling2023identifiability}}]\label{lemma_dettling}
Let $G=(V,E)$ be a directed graph with $|V|=d$. The graphical continuous Lyapunov model $\mathcal{M}_{G,\mathbf{C}}$ is generically identifiable if and only if there exists $\boldsymbol{\Sigma}\in\mathcal{M}_{G,\mathbf{C}}$ such that $\operatorname{rank}\!\left(\textbf{H}(\boldsymbol{\Sigma})_{E^c,\cdot}\right) = \frac{d(d-1)}{2}.
$
Equivalently, generic identifiability can be certified by exhibiting a nonzero $\frac{d(d-1)}{2}\times\frac{d(d-1)}{2}$ minor of $\textbf{H}(\boldsymbol{\Sigma})_{E^c,\cdot}$.
\end{lemma}
\paragraph{Proofs} Proofs of lemmas and theorems are given in the Appendix. 
\begin{figure}[t]
    \centering
    \begin{subfigure}[t]{0.28\linewidth}
        \centering     \includegraphics[page=1,width=\linewidth]{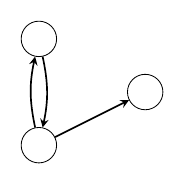}
        \caption{Child extension.}
    \end{subfigure}
    \begin{subfigure}[t]{0.28\linewidth}
        \centering       \includegraphics[page=2,width=\linewidth]{Figures/motifs_3v.pdf}
        \caption{Parent extension.}
    \end{subfigure}
    \caption{Generically identifiable $3$-node motifs.}
    \label{fig:thm1}
\end{figure}

\section{Generic Identifiability under Full Observability}
In this section, we investigate generic identifiability of the drift matrix under full observability. We first characterize a class of non-simple graphs that remain generically identifiable despite containing directed $2$-cycles (Theorem \ref{thm:3v_motif}). 
We then show a class of $4$-node motifs generically identifiable (Proposition \ref{prop:one-vertex-extensions}).

\subsection{Identifiability Beyond Simple Structures}
Generic identifiability of linear stochastic differential equations has been extensively studied through the algebraic characterization of the Lyapunov equation \ref{eq:lyapunov_stat}. \cite{dettling2023identifiability} showed that the drift matrix is globally identifiable if and only if the graph is simple, i.e., contains no $2$-cycles. They also observed that certain non-simple graphs may remain generically identifiable, illustrating this through isolated examples. However, no general graphical characterization of such non-simple identifiable graphs was provided. The goal of this section is to characterize a class of non-simple graphs for which generic identifiability is guaranteed.
\begin{theorem} \label{thm:3v_motif}
    Let $G=(V,E)$ be a directed graph satisfying: (i) Every directed 2-cycle $(i,j)$ is vertex-disjoint from all other 2-cycles; and (ii) For every directed $2$-cycle $(i,j)$, there exists a vertex $k_{ij}$ such that, after possibly exchanging $i$ and $j$, $k_{ij}\sim i$,$k_{ij}\not\sim j$, where $\sim$ denotes adjacency, i.e., there is an edge in at least one direction; and such that $k_{ij}$ does not belong to any directed $2$-cycle. Moreover, the vertices $k_{ij}$ can be chosen pairwise distinct for distinct directed $2$-cycles.
    $\mathcal{M}_{G, C}$ is generically identifiable.
\end{theorem}

\begin{proof}[Proof sketch]
    By Lemma~1, it suffices to exhibit $\Sigma\in\mathcal M_{G,C}$ for which
    $\mathbf H(\Sigma)_{E^c,\cdot}$ has full column rank.
    
    For each directed $2$-cycle $(i,j)$, let $k$ be its witness from assumption~2 and define the local block $N=\{i,j,k\}$. We specialize the drift matrix by retaining only the two edges of the $2$-cycle and the single directed edge between $k$ and its adjacent endpoint, and set all other allowed off-diagonal coefficients to zero. The retained off-diagonal coefficients are scaled by a sufficiently small $\varepsilon>0$. With negative diagonal coefficients, the resulting drift matrix is Hurwitz.
    
    Under this specialization, the stationary covariance is block diagonal across the sets $N$ and the remaining singleton vertices. Each $2$-cycle block contributes a nonsingular $3\times3$ minor of $\mathbf H(\Sigma)_{E^c,\cdot}$. For columns joining distinct blocks, at least one orientation of every vertex pair is a non-edge. Selecting one such non-edge row for every mixed column gives, as $\varepsilon\to0$, a diagonal minor with positive variance entries. Hence this mixed block is also nonsingular for sufficiently small $\varepsilon>0$.
    
    Combining the local and mixed minors gives a nonzero $\binom d2\times\binom d2$ minor of $\mathbf H(\Sigma)_{E^c,\cdot}$, and Lemma~1 yields generic
    identifiability.
\end{proof}

\subsection{Four-Node Non-Simple Identifiable Motifs}
Theorem \ref{thm:3v_motif} provides a sufficient condition for generic identifiability by requiring that every directed $2$-cycle has a neighboring node adjacent to exactly one of its nodes. This condition is not necessary. We now investigate four-node graphs that violate this condition yet remain generically identifiable.

Identifiability in the presence of a direct $2$-cycle is not governed by adjacency alone. A common parent, a common child, an asymmetric directed path, or an additional $2$-cycle may still fail to repair the rank deficiency created by a $2$-cycle. What matters is how the trek structure of the graph constrains the covariance entries in the matrix $\textbf{H}(\boldsymbol{\Sigma})_{E^c,\cdot}$. By Lemma \ref{lemma_dettling}, when $\lvert V\rvert = d = 4$, it suffices to construct a nonzero $6\times6$ minor of $\textbf{H}(\boldsymbol{\Sigma})_{E^c,\cdot}$.

\subsubsection{Extensions of deficient three-node motifs}

Figure~\ref{fig:three_node_motifs} shows the deficient three-node motifs from which we construct four-node extensions. Adding a vertex is not sufficient by itself; identifiability is restored only when the extension modifies the covariance structure so as to supply the missing independent constraint. The
following proposition summarizes the identifiable one-vertex extensions.
\begin{figure}[h]
    \centering
    \begin{subfigure}[t]{0.20\columnwidth}
        \centering
        \includegraphics[page=3,width=\linewidth]{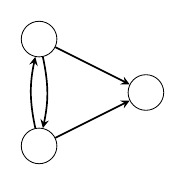}
        \caption{Common child.}
        \label{fig:motif3_a}
    \end{subfigure}
    \hfill
    \begin{subfigure}[t]{0.20\columnwidth}
        \centering
        \includegraphics[page=4,width=\linewidth]{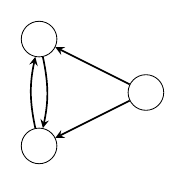}
        \caption{Common parent.}
        \label{fig:motif3_b}
    \end{subfigure}
    \hfill
    \begin{subfigure}[t]{0.21\columnwidth}
        \centering
        \includegraphics[page=5,width=\linewidth]{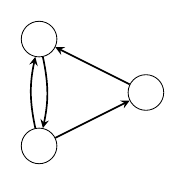}
        \caption{Mixed parent child.}
        \label{fig:motif3_c}
    \end{subfigure}
    \hfill
    \begin{subfigure}[t]{0.25\columnwidth}
        \centering
        \includegraphics[page=6,width=\linewidth]{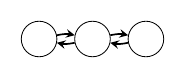}
        \caption{$2$-cycles intersection.}
        \label{fig:motif3_d}
    \end{subfigure}
    \caption{Three-node non-identifiable motifs.}
    \label{fig:three_node_motifs}
\end{figure}

\begin{proposition}[One-vertex extensions of deficient motifs]
    \label{prop:one-vertex-extensions}
    Let $G_0$ be one of the deficient three-node motifs in Figure~\ref{fig:three_node_motifs}, and let $G$ be obtained by adding a fourth node $v$ with exactly one adjacency to $G_0$. 

    The following extensions are generically identifiable: \(v\) is (i) a parent of a vertex of \(G_0\), except in the common-child setup (Figure \ref{fig:motif3_a}), where \(v\) may only be a parent of one vertex in a \(2\)-cycle; (ii) a child of a vertex of \(G_0\); or (iii) part of a directed \(2\)-cycle with a vertex of \(G_0\), except in the \(2\)-cycles intersection setup (Figure \ref{fig:motif3_d}).

\end{proposition}

These extensions illustrate that identifiability is governed not by adjacency alone, but by whether the added structure supplies new independent covariance constraints. Identifiable configurations, failed extensions, and their rank-based explanations are given in Appendix.

\subsection{Subgraph Identifiability}

The previous section studied identifiability of the entire drift matrix. We now consider subgraph identifiability, asking when the drift coefficients of an induced subgraph can be recovered even if the full graph is not identifiable.
\begin{theorem}[]
\label{thm:autonomous-subgraph}
Let \(G=(V,E)\) be a directed graph associated with the stationary linear SDE. Let \(S\subseteq V\)
Suppose that: (i) \(\mathcal{M}_{G[S],\textbf{C}_{SS}}\) is globally (respectively, generically) identifiable; and (ii) no vertex outside \(S\) is a parent of a vertex in \(S\), that is, $r\to s\notin E$  for every $r\in V\setminus S,\ s\in S.$
Then \(\textbf{M}_{SS}\) is globally (respectively, generically)
identifiable from \(\boldsymbol{\Sigma}_{SS}\), even if the complete drift matrix \(\textbf{M}\) is not identifiable.
\end{theorem}
\begin{proof}
Let \(R=V\setminus S\). The \((S,S)\)-block of the Lyapunov equation \ref{eq:lyapunov_stat} is
$
\textbf{M}_{SS}\boldsymbol{\Sigma}_{SS} + \textbf{M}_{SR}\boldsymbol{\Sigma}_{RS} + \boldsymbol{\Sigma}_{SS}\textbf{M}_{SS}^\top + \boldsymbol{\Sigma}_{SR}\textbf{M}_{SR}^\top + \textbf{C}_{SS}
=0.
$
By (ii)
, $\textbf{M}_{SR}=0$, and this reduces to
$
\textbf{M}_{SS}\boldsymbol{\Sigma}_{SS} + \boldsymbol{\Sigma}_{SS}\textbf{M}_{SS}^\top + \textbf{C}_{SS} =0.
$
This is  the Lyapunov equation associated with the induced
subgraph \(G[S]\). It involves only \(\textbf{M}_{SS}\), the observable covariance \(\boldsymbol{\Sigma}_{SS}\), and the known diffusion block \(\textbf{C}_{SS}\).

If the induced model is globally identifiable, the equation uniquely determines \(\mathbf{M}_{SS}\) from \(\boldsymbol{\Sigma}_{SS}\). For the generic case, let \(\mathcal{A}_S\) be the proper algebraic exceptional set on which the induced model may fail to be identifiable. Under the projection $\pi(\mathbf{M})=\mathbf{M}_{SS},$ the corresponding exceptional set in the full parameter space is $\mathcal{A}_G=\pi^{-1}(\mathcal{A}_S).$ This set is algebraic and proper: any stable \(\mathbf{M}_{SS}\notin\mathcal{A}_S\) can be extended to a stable full drift matrix of the form $\mathbf{M} = \begin{pmatrix} \mathbf{M}_{SS} & 0\\ \mathbf{M}_{RS} & \mathbf{M}_{RR} \end{pmatrix},$ by choosing \(\mathbf{M}_{RR}\) stable. Thus, outside the proper exceptional set \(\mathcal{A}_G\), the block \(\mathbf{M}_{SS}\) lies outside \(\mathcal{A}_S\) and is uniquely identified from \(\boldsymbol{\Sigma}_{SS}\).
\end{proof}

\section{Causal Effect Identification under Latent Confounding}

Previous section studied identification of the complete drift matrix under full observability. We now consider the same class of potentially non-simple graphs under partial observability, where directed 2-cycles 
remain allowed. Since latent variables prevent full drift recovery from the observed covariance, we focus on identifying a target coefficient from a rank condition. 
Let $G=(V,E)$ denote the graph associated with a stationary diffusion process and let $O\subseteq V$ be the set of observed variables. We write $L := V\setminus O$ for the latent variables. Herein, rather than recovering all model parameters, our objective is to identify causal effects from the observed stationary distribution.
\begin{lemma}\label{lemma2}
The retained covariance $\boldsymbol{\Sigma}_{OO}$ satisfies
$\textbf{M}^O\boldsymbol{\Sigma}_{OO}+\boldsymbol{\Sigma}_{OO}(\textbf{M}^O)^\top+\textbf{C}_{OO}=0$,
where $\textbf{M}^O:=\textbf{M}_{OO}+\textbf{M}_{OL}\boldsymbol{\Gamma}_{L\mid O}$ and $\boldsymbol{\Gamma}_{L\mid O}:=\boldsymbol{\Sigma}_{LO}\boldsymbol{\Sigma}^{-1}_{OO}$
\end{lemma}
By Lemma~\ref{lemma2}, the observed covariance matrix
$\boldsymbol{\Sigma}_{OO}$ satisfies a Lyapunov equation associated
with the effective observed drift matrix $\textbf{M}^O$.
When \((M^O)_{yx}\) is identifiable within the effective observed model, the relation below allows us to study when it coincides with the structural coefficient \(m_{yx}\).
For the target edge $x\to y$,
\[
(M^O)_{yx}
=
m_{yx}
+
\sum_{h\in L}m_{yh}(\Gamma_{L\mid O})_{hx}.
\]
Since $m_{yh}=0$ unless $h\in \operatorname{Pa}(y)$, we obtain $(M^O)_{yx}=m_{yx}$ if and only if $\sum_{h\in L\cap\operatorname{Pa}(y)}
m_{yh}(\Gamma_{L\mid O})_{hx}=0$. 
Therefore, either $m_{yh} = 0$ for all $h\in L$, i.e., $\operatorname{Pa}(y) \cap L = \emptyset$, or the summation vanishes, which is a non-generic condition because it requires an exact cancellation among the free parameters.
\begin{proposition}[Graphical sufficient condition]
Suppose that $(M^O)_{yx}$ is identifiable within the specified
effective observed model. If, for every omitted parent
$h\in L\cap\operatorname{Pa}(y)$, 
$X_h\perp X_x\mid X_{O\setminus\{x\}}$,
then
$(M^O)_{yx}=m_{yx}$,
and hence $m_{yx}$ is identifiable from
$(\boldsymbol{\Sigma}_{OO},\textbf{C}_{OO})$.
\end{proposition}
 
\begin{proof}
The condition implies $(\Gamma_{L\mid O})_{hx}=0$ for every $h\in L\cap \operatorname{Pa}(y)$.
Therefore the bias term in the expression for $(M^O)_{yx}$ vanishes, giving $(M^O)_{yx} = m_{yx}.$
\end{proof}

This condition is useful but limited. When $h\to y\leftarrow x$, the trek-graph separation condition does not hold, since the path $h-y-x$ passes through the observed node $y\notin L\setminus{h}$. Therefore, the graphical criterion does not guarantee $X_h\perp X_x\mid X_{O\setminus{x}}$. This motivates the instrumental-variable approach below.

\subsection{Instrumental Variables} 
Let $P =  \operatorname{Pa}(y) \cap O$, $U = \operatorname{Pa}(y) \cap L$. We represent every diagonal drift coefficient as a self-loop. Consequently, \(y\in\operatorname{Pa}(y)\) and therefore \(y\in P\). The $y$-row of the drift is
\[ 
\sum _{p \in P}m_{yp}X_p + \sum _{u \in U}m_{yu}X_u.
\]
When $U \neq \emptyset$, the observed parent set $P$ is generally not a valid adjustment set because marginalization introduces the bias term derived in the previous section. We therefore seek to recover the direct drift coefficients $m_{yp}$ from observational covariances using instrumental variables. Let $Z=\{z_1,\dots,z_k\}\subseteq O\setminus \operatorname{Pa}(y)$ be a set of observed instrumental variables. In particular, $P\cap Z=\emptyset$, so no instrument has a direct edge into $y$.

\begin{proposition}[Instrumental equations]
    For every \(z\in Z\), the \((y,z)\)-entry of the Lyapunov equation yields
    \[
    \sum_{p\in P} m_{yp}\Sigma_{pz}
    +\sum_{u\in U} m_{yu}\Sigma_{uz}
    +\sum_{q\in \operatorname{Pa}(z)} m_{zq}\Sigma_{yq}= -c_{yz},
    \]
    since \(\mathbf{C}\) is diagonal and \(z\neq y\), so \(c_{yz}=0\).
    Here, \(q\) ranges over all parents of \(z\).
\end{proposition}

The last summation depends on the incoming coefficients of $z$. Let $b_{z} := \sum_{q \in \operatorname{Pa}(z)}m_{zq}\Sigma_{yq}$. 
\begin{lemma}[Root instruments]
    Suppose $z \in O$ has no incoming edges except its self-loop. Then $b_{z} = m_{zz}\Sigma_{yz}$, where 
    $
    m_{zz} = -\frac{c_{zz}}{2\Sigma_{zz}},
    $
    is identifiable from the $(z, z)$-entry of the Lyapunov equation.
\end{lemma}

Throughout the remainder of this section, we assume that every instrument is a root node. We now impose an instrumental validity condition. The hidden parents of $y$ must be marginally independent with the instruments, that is, $\Sigma_{uz} = 0$ for all $u\in U, z \in Z$. Graphically, this condition holds generically when there is no trek between $u$ and $z$. Conversely, if a trek between $u$ and $z$ exists, then $\Sigma_{uz}$ is generically nonzero.
\begin{lemma}[Linear identification system]
\label{lemma:linearidsys}
    Suppose $\Sigma_{uz} = 0$, for all $u\in U, z\in Z$, and suppose $b_z$ is known for every $z \in Z$. Then $\sum_{p\in P}m_{yp}\Sigma_{pz} = -b_z$.

    Collecting all instruments gives
    $\boldsymbol{\Sigma}^\top_{PZ}m_{yP} = -b, \quad b = (b_{z_1}, \cdots, b_{z_k})^\top$. The unknown vector is $m_{yP} = (m_{yp})_{p \in P} \in \mathbb{R}^{\lvert P \rvert}$.
\end{lemma}
\begin{proof}
    From the $(y,z)$-entry of the Lyapunov equation, the instrumental validity assumption implies that $\Sigma_{uz} = 0$ for every $u \in U, z \in Z$, vanishing the hidden-parent terms. Moreover, because of the root instruments, the remaining known contribution is captured by $b$. Hence, $\sum_{p \in P} m_{yp}\Sigma_{pz} + b_z = 0$.
\end{proof}

Since $\operatorname{rank}(\boldsymbol{\Sigma}_{PZ})\leq \min\{|P|,|Z|\}$, full row rank requires $|Z|\geq |P|$. Recall that $P$ includes $y$ itself because the diagonal drift coefficient $m_{yy}$ is treated as a self-loop coefficient.
\begin{theorem}[Identification from instrumental variables]
\label{thm:ivthm3}
Assume that $\Sigma_{uz}=0$ for every $u\in U$ and $z\in Z$,
$b_z$ is known for every $z\in Z$,
$\operatorname{rank}(\boldsymbol{\Sigma}_{PZ})=|P|$,
and $P\cap Z=\emptyset$.
Then $m_{yP}$ is uniquely identifiable from observational covariance.
In particular, $m_{yx}$ is identifiable for every $x\in P$.
\end{theorem}
\begin{proof}
    By Lemma \ref{lemma:linearidsys}, $m_{yP}$ satisfies $\boldsymbol{\Sigma}_{PZ}^\top m_{yP} = - b$. The matrix $\boldsymbol{\Sigma}_{PZ} \in \mathbb{R}^{\lvert P \rvert \times \lvert Z \rvert}$. Since $\operatorname{rank}(\boldsymbol{\Sigma}_{PZ}) = \lvert P \rvert$, the matrix $\boldsymbol{\Sigma}_{PZ}^\top$ has full column rank. Therefore 
    $\boldsymbol{\Sigma}_{PZ}\boldsymbol{\Sigma}_{PZ}^\top m_{yP} = - \boldsymbol{\Sigma}_{PZ}b,$ where $\boldsymbol{\Sigma}_{PZ}\boldsymbol{\Sigma}_{PZ}^\top$ is invertible.
    Thus, the linear system has at most one solution for $m_{yP}$. Since the true structural coefficient vector satisfies the system, the solution is unique. Thus every coefficient $m_{yp}$, $p \in P$, is identifiable. In particular, $m_{yx}$ is identifiable. 
\end{proof}

We now aim to give a graphical condition ensuring the rank assumption. 

\subsection{Graphical Conditions for the Instrument Rank Assumption}
The identification result above requires $\operatorname{rank}(\boldsymbol{\Sigma}_{PZ})=|P|$. We now provide graphical conditions under which the rank assumption holds generically. The structure of the result depends on whether the instruments are directly connected to $P\setminus\{y\}$ or whether directed paths contain intermediate vertices.

When there are no intermediate vertices, the problem reduces to a bipartite matching condition and admits an exact characterization through Hall's theorem \citep{hall1987representatives}. With intermediate vertices, a system of vertex-disjoint directed paths provides a sufficient condition for generic full rank. The latter condition is not necessary in general.

\subsubsection{Direct Instruments}
We first consider the case in which every directed path from an instrument to $P\setminus\{y\}$ consists of a single edge.
\begin{theorem}[Graphical rank characterization without intermediate vertices]\label{thm:directIV}
    Let $P' = \{p_1, \cdots, p_r\}=P\setminus\{y\}$ and let $Z$ be a set of root instruments. Assume that no $z \in Z$ has a direct edge to $y$, and every directed path from $Z$ to $P'$ consists of a single edge $z \to p$. Define the bipartite graph \(B=(Z,P',E_B)\), where \((z,p)\in E_B\iff z\to p\in E\).
    
    Then, generically, $\operatorname{rank}(\boldsymbol{\Sigma}_{PZ})=|P|$ if and only if there exist a parent $a \in P'$ and two distinct instruments $z', z'' \in Z$ such that (i) $z' \to a$ and $z'' \to a$ with $m_{z'z'} \neq m_{z''z''}$; and (ii) after deleting $a$, $z'$, and $z''$ from $B$, the remaining bipartite graph contains a matching covering $P'\setminus\{a\}$.
\end{theorem}
\begin{proof}[Proof sketch]
For sufficiency, we specialize the model parameters according to the prescribed matching and the two instruments adjacent to $a \in P'$, and show that the resulting covariance matrix contains a nonzero $|P|\times|P|$ minor. The matching guarantees independent contributions from all parents except $a$, while the Lyapunov equation shows that the two instruments incident to $a$ contribute an additional independent direction whenever $m_{z'z'}\neq m_{z''z''}$.
For necessity, generic full rank of $\boldsymbol{\Sigma}_{PZ}$ first implies that $\boldsymbol{\Sigma}_{P'Z}$ has generic full rank, which is equivalent to the existence of a matching covering $P'$ in the bipartite graph $B$. The remaining rank contribution must then come from the row indexed by $y$, and the Lyapunov equation shows that this is possible only if two instruments share a parent and have distinct self-drifts. 
\end{proof}

\subsubsection{Intermediate Vertices}
When directed paths from $Z$ to $P\setminus\{y\}$ contain intermediate vertices, the support pattern of $\boldsymbol{\Sigma}_{PZ}$ is no longer represented by a
bipartite graph alone. Nevertheless, a system of vertex-disjoint directed paths provides a sufficient graphical condition for generic full rank.
\begin{theorem}[Sufficient graphical rank condition with intermediate nodes]\label{thm:nondirectIV}
    Let $P' = \{p_1, \cdots, p_r\}=P\setminus\{y\}$ and let $Z$ be a set of root instruments. Assume that no $z \in Z$ has a direct edge to $y$. 

    Then, generically, $\operatorname{rank}(\boldsymbol{\Sigma}_{PZ})=|P|$ if there exist a parent $a \in P'$ and two distinct instruments $z', z'' \in Z$ such that (i) $\pi': z' \rightsquigarrow a$ and $\pi'': z'' \rightsquigarrow a$ with $m_{z'z'} \neq m_{z''z''}$; (ii) there exist distinct instruments $w_1, \cdots, w_{r-1} \in Z \setminus \{z', z''\}$ and pairwise vertex-disjoint direct paths $\rho_i:w_i \rightsquigarrow q_i$ for $i = 1, \cdots, r-1$ whose endpoints satisfy $\{q_1, \cdots q_{r-1}\} = P'\setminus \{a\}$; and (iii) each $\rho_i$ is vertex-disjoint from $\pi' \cup \pi''$.
\end{theorem}

\begin{proof}[Proof sketch]
The proof follows the same strategy as Theorem~\ref{thm:directIV}. We specialize the model parameters so that only the prescribed paths are active. The vertex-disjointness assumptions isolate the covariance contributions of the paths $\rho_1,\ldots,\rho_{r-1}$, yielding independent covariance terms for the parents in $P'\setminus\{a\}$, while ensuring that these paths do not interfere with the two paths terminating at $a$. The remaining $2\times2$ covariance block associated with $a$, $y$, and the instruments $z',z''$ is identical to that of Theorem~\ref{thm:directIV} and is nonsingular whenever $m_{z'z'}\neq m_{z''z''}$. Consequently, the selected $|P|\times|P|$ covariance minor is generically nonzero, implying $\operatorname{rank}(\boldsymbol{\Sigma}_{PZ})=|P|$. 
\end{proof}

\begin{remark}
The number of hidden parents of $y$ does not directly affect identifiability. Hidden parents may be arbitrarily many, provided that they are marginally uncorrelated with the instruments. Thus, the relevant obstruction is not the size of $U$, but the existence of treks between $U$ and $Z$.
\end{remark}

\begin{remark}
The instrumental-variable results do not require the observed subgraph to be acyclic or simple. Reciprocal interactions may occur among the target, its observed parents, and intermediate observed vertices, provided that the instruments remain root vertices and the required trek-separation conditions between the instruments and latent parents are satisfied.
\end{remark}

\subsubsection{Bottleneck Investigation}

Theorem~\ref{thm:nondirectIV} gives a sufficient, but not necessary, graphical condition for generic full rank of $\boldsymbol{\Sigma}_{PZ}$ based on vertex-disjoint paths. 
We investigate whether recovery remains possible when the number of intermediate nodes is insufficient to satisfy this condition.

For each graph configuration, we generate 100 random stable drift matrices with diagonal diffusion matrices and compute the stationary covariance matrix directly from the Lyapunov equation. The generated graphs have a layered structure in which every directed path from $Z$ to $P'$ passes through the set $Q$ of intermediate nodes, with no direct edges from $Z$ to $P'$. Hence, whenever $|Q|<|P'|-1$, the $|P'|-1$ pairwise vertex-disjoint paths required by Theorem~\ref{thm:nondirectIV} cannot exist. We vary the bottleneck ratio ${|Q|}/{|P'|}$, recover $\hat m_{yP'}$ from $\boldsymbol{\Sigma}_{PZ}$, and report the mean relative error $\|\hat m_{yP'}-m_{yP'}\|_2/\|m_{yP'}\|_2$. The graph generation procedure, parameter initializations, and implementation are provided in the Appendix.

Figure~\ref{fig:bottleneck} shows that accurate recovery is possible even when $|Q|<|P'|-1$, where the vertex-disjoint-path condition of Theorem~\ref{thm:nondirectIV} is guaranteed to fail. This demonstrates that the sufficient graphical condition of Theorem~\ref{thm:nondirectIV} is not necessary for full-rank recovery. Moreover, larger parent sets require a higher bottleneck ratio before the estimation error decreases substantially, suggesting that recovering more coefficients requires proportionally more intermediate vertices.

\begin{figure}
    \centering
    \includegraphics[width=0.8\linewidth]{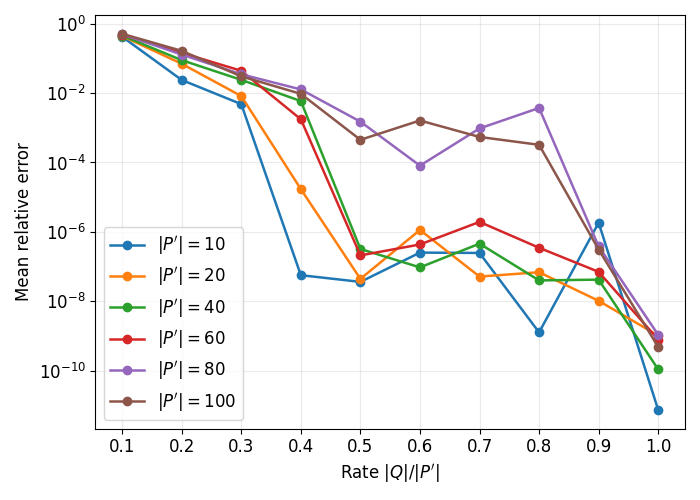}
    \caption{Recovery error of the target drift coefficients as a function of the bottleneck ratio $ {\vert Q \rvert}/{\vert P' \rvert}$. $Q$ is the intermediate nodes and $P' = P \setminus\{y\}$
    . Curves correspond to different numbers of parents. Latent parents $\lvert U \rvert = 5$.}
    \label{fig:bottleneck}
\end{figure}

\section{Restrictiveness of Theorem 1}

We empirically assess how restrictive the sufficient  graphical condition of Theorem~1 is relative to the kernel-rank criterion of Lemma~1; experimental details and the full results are reported in Appendix. Among non-simple graphs with numerical full kernel rank, Theorem~1 certifies $94.4$--$97.8\%$ of cases in the sparsest regime considered (edge probability $0.05$). Coverage decreases with graph density, reaching $6.4\%$ for $|V|=12$ at edge probability $0.30$. This decline is expected because denser graphs contain more overlapping 2-cycles, which increasingly violate the vertex-disjoint structure required by Theorem~1, even when the kernel matrix remains full rank. Thus, the theorem provides broad coverage in sparse regimes while leaving additional identifiable structures beyond its sufficient conditions in denser graphs.

\section{Conclusion and Future Work}

We studied identification from stationary covariance in non-simple linear SDEs through a common rank-based perspective. Directed 2-cycles are allowed in both the fully and partially observed settings. We established graphical conditions for generic identifiability in the presence of directed $2$-cycles, developed an instrumental-variable framework for identifying direct drift coefficients under latent confounding, and derived graphical conditions guaranteeing the required instrumental covariance rank. As this work focuses on the theoretical problem of identifiability, an important direction for future work is to investigate the statistical properties of the proposed framework under finite samples, including consistency, convergence rates, and robustness to covariance estimation error. Additional directions include extending the graphical identifiability results to broader classes of reciprocal structures and relaxing assumptions on the diffusion matrix.



\bibliography{AISTATS}

\clearpage
\appendix
\thispagestyle{empty}

\onecolumn
\aistatstitle{Appendix}

\section{Proofs}

\subsection{Proof of Theorem \ref{thm:3v_motif}}

\begin{proof}

By Lemma~1, it suffices to show that there exists
\(\boldsymbol{\Sigma}\in\mathcal M_{G,\mathbf C}\) such that
\(\operatorname{rank}\left(\mathbf H(\boldsymbol{\Sigma})_{E^c,\cdot}\right)=\binom d2\), where \(|V|=d\).

Let the directed \(2\)-cycles of \(G\) be
\(c_m=\{i_m,j_m\}\), \(m=1,\ldots,s\),
and let \(k_m\) denote their witnesses from the assumption of the
theorem. Thus, after possibly exchanging \(i_m\) and \(j_m\),
\(k_m\sim i_m\) and \(k_m\not\sim j_m\),
each \(k_m\) belongs to no directed \(2\)-cycle, and the vertices
\(k_1,\ldots,k_s\) are pairwise distinct.

Define \(N_m:=\{i_m,j_m,k_m\}\) and
\(R:=V\setminus\bigcup_{m=1}^sN_m\).
By the assumptions of the theorem, the sets
\(N_1,\ldots,N_s,\{r\}_{r\in R}\)
are pairwise disjoint and form a partition of \(V\).

\paragraph{Parameter specialization.}

For every vertex \(v\in V\), choose
\(m_{vv}=-\alpha_v\), with \(\alpha_v>0\).

For each neighborhood \(N_m=\{i_m,j_m,k_m\}\), retain all three edges
\(i_m\to j_m\), \(j_m\to i_m\),
and the unique directed edge between \(i_m\) and \(k_m\).
Choose all three corresponding drift coefficients strictly
positive and nonzero.

Set all remaining off-diagonal coefficients equal to zero. In
particular, this disconnects distinct neighborhoods from one
another in the chosen parameter specialization, but both edges of
every directed \(2\)-cycle remain nonzero.

For a fixed neighborhood \(N_m=\{i,j,k\}\), write
\(m_{ii}=-a\), \(m_{jj}=-b\), \(m_{kk}=-c\), with \(a,b,c>0\),
and denote the two coefficients of the directed \(2\)-cycle by
\(m_{ji}=p>0\) and \(m_{ij}=q>0\).
Thus \(p\) corresponds to \(i\to j\) and \(q\) to
\(j\to i\).

Choose these coefficients such that \(pq<ab\).
The \(2\times2\) drift block on \(\{i,j\}\) is then Hurwitz.
Since the witness edge is one-directional, the \(3\times3\)
block on \(N_m\) is block triangular after an appropriate
ordering, with the eigenvalue \(-c\) in addition to the two
eigenvalues of the \(2\times2\) cycle block. Hence every
neighborhood block is Hurwitz. Therefore the full specialized
matrix \(\mathbf M\) is Hurwitz.

Since \(\mathbf C\) is diagonal positive definite, the Lyapunov
equation
\(\mathbf M\boldsymbol{\Sigma}+\boldsymbol{\Sigma}\mathbf M^\top+\mathbf C=0\)
has a unique positive definite solution
\(\boldsymbol{\Sigma}\).

Moreover, \(\mathbf M\) and \(\mathbf C\) are block diagonal
with respect to
\(N_1,\ldots,N_s,\{r\}_{r\in R}\).
By uniqueness of the Lyapunov solution,
\(\boldsymbol{\Sigma}\) is block diagonal with respect to the
same partition. Consequently, \(\Sigma_{uv}=0\)
whenever \(u\) and \(v\) belong to distinct blocks.

We choose the positive off-diagonal coefficients
\(p,q\), and the witness-edge coefficient in every neighborhood,
sufficiently small that each covariance block
\(\boldsymbol{\Sigma}_{N_mN_m}\) is strictly diagonally
dominant:
\[
\Sigma_{vv}
>
\sum_{\substack{u\in N_m\\u\neq v}}
|\Sigma_{uv}|,
\qquad
v\in N_m.
\]
Such a choice is possible because the stationary covariance
depends continuously on the drift coefficients and can be made
arbitrarily close to a positive diagonal matrix by choosing all
off-diagonal coefficients sufficiently small. Importantly, all
coefficients \(p,q\) are chosen strictly positive; no directed
\(2\)-cycle is removed.

We will also use that \(\Sigma_{ik}>0\) for every neighborhood.
Indeed, the specialized drift matrix is Metzler, since all retained off-diagonal coefficients are
positive. Hence \(e^{t\mathbf M}\geq0\) entrywise for \(t\geq0\).
Moreover, because there is a direct
positive edge between \(i\) and \(k\), the corresponding entry of
\(e^{t\mathbf M}\) is strictly positive for \(t>0\). Using the
stationary covariance representation
\[
\boldsymbol{\Sigma}
=
\int_0^\infty
e^{t\mathbf M}\mathbf C e^{t\mathbf M^\top}\,dt,
\]
with \(\mathbf C\) diagonal positive definite, it follows that
\(\Sigma_{ik}>0\).

\paragraph{Local blocks.}

Fix one directed \(2\)-cycle \(i\leftrightarrow j\)
with witness \(k\), where \(k\sim i\) and \(k\not\sim j\).

Consider the three covariance columns
\((i,j),(i,k),(j,k)\).
The displayed matrices use these oriented column labels for
notational convenience. Changing the fixed ordering convention
for an unordered covariance column only changes its sign and
therefore does not affect nonsingularity.

Suppose first that \(i\to k\).
Then \(k\to i\), \(k\to j\), and \(j\to k\)
are non-edges. Selecting the rows
\(j\to k\), \(k\to i\), and \(k\to j\)
gives, up to harmless row and column signs,
\[
\mathbf H_m
=
\begin{pmatrix}
\Sigma_{ik} & 0 & -\Sigma_{kk}\\
0 & \Sigma_{ii} & \Sigma_{ij}\\
0 & \Sigma_{ij} & \Sigma_{jj}
\end{pmatrix}.
\]
Therefore
\[
\det(\mathbf H_m)
=
\Sigma_{ik}
\left(
\Sigma_{ii}\Sigma_{jj}-\Sigma_{ij}^2
\right).
\]
Since \(\boldsymbol{\Sigma}\succ0\), the principal \(2\times2\) minor satisfies
\(\Sigma_{ii}\Sigma_{jj}-\Sigma_{ij}^2>0\).
Moreover, \(\Sigma_{ik}>0\).
Thus \(\det(\mathbf H_m)\neq0\).

Now suppose that \(k\to i\).
Then \(i\to k\), \(k\to j\), and \(j\to k\)
are non-edges. Selecting these three rows gives
\[
\mathbf H_m
=
\begin{pmatrix}
-\Sigma_{jk} & -\Sigma_{kk} & 0\\
0 & \Sigma_{ij} & \Sigma_{jj}\\
\Sigma_{ik} & 0 & -\Sigma_{kk}
\end{pmatrix}.
\]
Hence
\[
\det(\mathbf H_m)
=
\Sigma_{kk}
\left(
\Sigma_{ij}\Sigma_{jk}
-
\Sigma_{ik}\Sigma_{jj}
\right).
\]
We now show directly from the Lyapunov equation that the
expression in parentheses is nonzero.

Recall that \(m_{ji}=p>0\)
is the coefficient corresponding to the edge \(i\to j\).
Since \(k\to i\) and there is no edge between \(j\) and \(k\),
the \((j,k)\)-entry of the Lyapunov equation gives
\(p\Sigma_{ik}-(b+c)\Sigma_{jk}=0\).
Therefore
\(\Sigma_{jk}=\frac{p}{b+c}\Sigma_{ik}\).

The \((j,j)\)-entry of the Lyapunov equation gives
\(2p\Sigma_{ij}-2b\Sigma_{jj}+C_{jj}=0\),
and hence
\(p\Sigma_{ij}=b\Sigma_{jj}-\frac{C_{jj}}{2}\).

Consequently,
\[
\begin{aligned}
\Sigma_{ij}\Sigma_{jk}
-
\Sigma_{ik}\Sigma_{jj}
&=
\Sigma_{ik}
\left(
\frac{p\Sigma_{ij}}{b+c}
-
\Sigma_{jj}
\right)
\\
&=
\Sigma_{ik}
\left(
\frac{
b\Sigma_{jj}-C_{jj}/2
}{
b+c
}
-
\Sigma_{jj}
\right)
\\
&=
-
\frac{\Sigma_{ik}}
{b+c}
\left(
c\Sigma_{jj}
+
\frac{C_{jj}}{2}
\right).
\end{aligned}
\]
Since
\(\Sigma_{ik}>0\), \(\Sigma_{jj}>0\), \(C_{jj}>0\), and \(b+c>0\),
we obtain
\(\Sigma_{ij}\Sigma_{jk}-\Sigma_{ik}\Sigma_{jj}<0\).
Therefore \(\det(\mathbf H_m)\neq0\).

Thus, in both possible orientations of the witness edge, the
three local columns admit a nonsingular \(3\times3\) minor while
both edges of the directed \(2\)-cycle remain nonzero.

\paragraph{Mixed blocks.}

It remains to consider covariance columns whose vertices belong to
distinct blocks among \(N_1,\ldots,N_s,\{r\}_{r\in R}\).

Let \(A\) and \(B\) be two distinct such blocks. Consider all
columns \((a,b)\), \(a\in A\), \(b\in B\).
There are \(|A||B|\) such columns.

For every \(a\in A\) and \(b\in B\), at least one of
\(a\to b\) and \(b\to a\)
belongs to \(E^c\). Otherwise \(a\) and \(b\) would themselves
form a directed \(2\)-cycle. This would contradict either the
vertex-disjointness of the directed \(2\)-cycles, the assumption
that no witness \(k_m\) belongs to a directed \(2\)-cycle, or
the definition of \(R\).

For every mixed column \((a,b)\), select one available non-edge
row, either \(a\to b\) or \(b\to a\). Distinct mixed columns
give distinct selected rows. Denote the resulting square
submatrix by \(\mathbf H_{AB}\).
We show that \(\mathbf H_{AB}\) is strictly diagonally
dominant.

Suppose the row selected for column \((a,b)\) is \(a\to b\).
Its entry in the corresponding column is, up to sign,
\(\Sigma_{bb}\).
Within \(\mathbf H_{AB}\), this row may also contain nonzero
entries in columns \((a,c)\), \(c\in B\), \(c\neq b\).
These entries are, up to sign, \(\Sigma_{cb}\).
In particular, these terms need not vanish when \(b\) and \(c\)
belong to the same neighborhood \(N_m\).

Nevertheless, by the strict diagonal dominance of
\(\boldsymbol{\Sigma}_{BB}\),
\[
\Sigma_{bb}
>
\sum_{\substack{c\in B\\c\neq b}}
|\Sigma_{cb}|.
\]
Hence this row of \(\mathbf H_{AB}\) is strictly diagonally
dominant.

If instead the selected row is \(b\to a\),
its diagonal entry is, up to sign, \(\Sigma_{aa}\),
and its possible off-diagonal entries have the form
\(\Sigma_{ca}\), \(c\in A\), \(c\neq a\).
Therefore,
\[
\Sigma_{aa}
>
\sum_{\substack{c\in A\\c\neq a}}
|\Sigma_{ca}|.
\]
This row is also strictly diagonally dominant.

Thus every row of \(\mathbf H_{AB}\) is strictly diagonally
dominant. By the Levy--Desplanques theorem \cite{horn2012matrix},
\(\mathbf H_{AB}\) is nonsingular.

This argument includes the case in which one block is a singleton
\(\{r\}\subseteq R\). If both blocks are singletons, the
corresponding block is simply \(1\times1\) with nonzero entry
equal, up to sign, to a positive variance.

\paragraph{Combining the blocks.}

For each \(N_m\), select the three local rows defining
\(\mathbf H_m\). For each unordered pair of distinct blocks
\(A,B\), select the rows defining \(\mathbf H_{AB}\).

Together these rows correspond to all \(\binom d2\)
covariance columns and form a square
\(\binom d2\times\binom d2\)
submatrix of
\(\mathbf H(\boldsymbol{\Sigma})_{E^c,\cdot}\).

After suitable row and column permutations, this selected matrix
is block diagonal.

Indeed, a row selected for a pair of blocks \(A,B\) can have a
nonzero entry only in columns corresponding to the same pair of
blocks. Any apparent entry involving a third block contains a
covariance between distinct covariance blocks and therefore
vanishes because \(\boldsymbol{\Sigma}\) is block diagonal.
Similarly, mixed rows vanish on local columns, and local rows
associated with \(N_m\) vanish on columns associated with
\(N_\ell\), \(m\neq\ell\).

Thus the selected minor has the form
\[
\operatorname{diag}
\left(
\mathbf H_1,\ldots,\mathbf H_s,
\{\mathbf H_{AB}:A\neq B\}
\right).
\]
Every local block \(\mathbf H_m\) is nonsingular, and every
mixed block \(\mathbf H_{AB}\) is nonsingular. Hence the entire
selected minor is nonsingular.

Therefore,
\(\operatorname{rank}\left(\mathbf H(\boldsymbol{\Sigma})_{E^c,\cdot}\right)
=\binom d2=\frac{d(d-1)}2\).
By Lemma~1, \(\mathcal M_{G,\mathbf C}\) is generically identifiable.
\end{proof}

\subsection{Proof of Proposition \ref{prop:one-vertex-extensions}}

\begin{proof}
Let \(G\) be one of the four-node extensions listed in
Proposition~1. Since \(d=4\), the matrix
\(\mathbf H(\boldsymbol{\Sigma})_{E^c,\cdot}\) has
\(\binom{4}{2}=6\) columns. By Lemma~1, it is sufficient to
find one covariance matrix
\(\boldsymbol{\Sigma}\in\mathcal M_{G,C}\) such that
\(\operatorname{rank}\bigl(\mathbf H(\boldsymbol{\Sigma})_{E^c,\cdot}\bigr)=6\).
Equivalently, it is enough to exhibit a nonsingular
\(6\times6\) minor of
\(\mathbf H(\boldsymbol{\Sigma})_{E^c,\cdot}\).

Some extensions satisfy the assumptions of Theorem~1
directly. In particular, if the added vertex is adjacent to
exactly one endpoint of an isolated directed \(2\)-cycle,
generic identifiability follows immediately from Theorem~1.
We therefore only consider the extensions that are not covered
by that result.

For each remaining graph, we select six rows indexed by
non-edges and consider the determinant of the corresponding
\(6\times6\) minor. Since every entry of
\(\mathbf H(\boldsymbol{\Sigma})\) is linear in the entries of
\(\boldsymbol{\Sigma}\), this determinant is a polynomial in
the covariance entries. Therefore, it is sufficient to show
that it is nonzero for one admissible covariance matrix. This
proves that the determinant is not identically zero and hence
that the selected rows are generically linearly independent.

\paragraph{Common-child extension.}

Label the common child by \(1\), the vertices of the directed
\(2\)-cycle by \(2\) and \(3\), and the added vertex by
\(4\). Thus, the original motif contains
\(2\to1\), \(3\to1\), \(2\to3\), and \(3\to2\).
Order the columns of the kernel matrix as
\((1,2),(1,3),(1,4),(2,3),(2,4),(3,4)\).

Consider first the extension \(1\to4\). Using the non-edge
rows \(1\to2\), \(1\to3\), \(4\to2\), \(4\to3\),
\(2\to4\), and \(3\to4\), we obtain the minor
\[
K_{\mathrm{cc}}(\boldsymbol{\Sigma})
=
\begin{pmatrix}
-\Sigma_{22}&-\Sigma_{23}&-\Sigma_{24}&0&0&0\\
-\Sigma_{23}&-\Sigma_{33}&-\Sigma_{34}&0&0&0\\
0&0&\Sigma_{12}&0&\Sigma_{22}&\Sigma_{23}\\
0&0&\Sigma_{13}&0&\Sigma_{23}&\Sigma_{33}\\
\Sigma_{14}&0&0&-\Sigma_{34}&-\Sigma_{44}&0\\
0&\Sigma_{14}&0&\Sigma_{24}&0&-\Sigma_{44}
\end{pmatrix}.
\]

The directed paths \(2\to1\to4\) and \(3\to1\to4\)
are treks. Therefore, \(\Sigma_{24}\) and \(\Sigma_{34}\)
are not structurally forced to vanish.

To give an exact rank certificate, take \(C=I_4\) and
\[
M^\star=
\begin{pmatrix}
-4&1&1&0\\
0&-5&1&0\\
0&1&-6&0\\
1&0&0&-7
\end{pmatrix}.
\]
This matrix has the required zero pattern and is Hurwitz by
strict diagonal dominance. Let
\(\boldsymbol{\Sigma}^\star\) be the unique solution of
\(M^\star\boldsymbol{\Sigma}^\star+
\boldsymbol{\Sigma}^\star(M^\star)^\top+I_4=0\).
Exact evaluation gives
\(\det K_{\mathrm{cc}}(\boldsymbol{\Sigma}^\star)
=
23289823567/15776105854472994800\neq0\).
Hence the determinant of
\(K_{\mathrm{cc}}(\boldsymbol{\Sigma})\) is not identically
zero. The six selected rows are therefore generically
linearly independent, and the graph with \(1\to4\) is
generically identifiable.

The optional reverse edge \(4\to1\) does not remove any of
the selected non-edge rows. Moreover, because
\(\mathcal M_{G,C}\) is defined by
\(\operatorname{supp}(M)\subseteq E\), the same parameter
point remains admissible for the graph containing
\(1\to4\) and \(4\to1\) by setting the coefficient associated
with the optional edge equal to zero. Thus the same minor
also proves generic identifiability when both directions are
present.

The reverse orientation \(4\to1\) alone is excluded. Indeed,
the paths \(4\to1\leftarrow2\) and
\(4\to1\leftarrow3\) contain a collider at \(1\).
Hence there is no trek from \(4\) to either \(2\) or \(3\),
and the trek rule gives \(\Sigma_{24}=\Sigma_{34}=0\).
The column of
\(\mathbf H(\boldsymbol{\Sigma})_{E^c,\cdot}\) indexed by
\((2,3)\) is then identically zero: its only potentially
nonzero entries occur in the rows \(2\to4\) and \(3\to4\),
where they are respectively
\(-\Sigma_{34}\) and \(\Sigma_{24}\). Consequently,
\(\operatorname{rank}\bigl(
\mathbf H(\boldsymbol{\Sigma})_{E^c,\cdot}
\bigr)<6\).
This explains the exception in the common-child case.

\paragraph{Remaining extensions.}

The other extensions are proved by the same rank argument.
For each labelled configuration \(G\), the exact certificate
is constructed computationally in the supplemental code.
Specifically, the code constructs a rational Hurwitz drift
matrix \(M_G^\star\) with the required graph support, sets
\(C=I_4\), solves the Lyapunov equation exactly over the
rationals to obtain \(\boldsymbol{\Sigma}_G^\star\), constructs
\(\mathbf H(\boldsymbol{\Sigma}_G^\star)_{E^c,\cdot}\), and
selects six non-edge rows forming a nonsingular \(6\times6\)
minor. Thus, for the selected row set \(R_G\), with
\(K_G(\boldsymbol{\Sigma})
=
\mathbf H(\boldsymbol{\Sigma})_{R_G,\cdot}\), the supplemental
code provides an exact certificate satisfying
\(\det K_G(\boldsymbol{\Sigma}_G^\star)\neq0\).

Therefore, in each case, the determinant is not identically
zero and the selected six rows are generically linearly
independent. The supplemental material contains the
corresponding parameter specialization, selected non-edge
rows, and exact nonzero determinant for each labelled graph.

Configurations obtained by relabeling vertices only permute
the rows and columns of the selected minor and therefore
preserve its rank. Similarly, when an indicated reverse edge
is optional, adding that edge does not remove the rows of the
original certificate, and the same parameter specialization
remains admissible after assigning zero to the optional
coefficient.

Thus, for every extension listed in Proposition~1, there
exists
\(\boldsymbol{\Sigma}\in\mathcal M_{\mathbf{G,C}}\) such that
\(\operatorname{rank}\bigl(
\mathbf H(\boldsymbol{\Sigma})_{E^c,\cdot}
\bigr)=6\).
The conclusion follows from Lemma~1.
\end{proof}

\subsection{Proof of Lemma \ref{lemma2}}

\begin{proof}
Partition the drift matrix and stationary covariance according to the observed
and hidden variables:
\[
\textbf{M}=
\begin{pmatrix}
\textbf{M}_{OO} & \textbf{M}_{OL}\\
\textbf{M}_{LO} & \textbf{M}_{LL}
\end{pmatrix},
\qquad
\boldsymbol{\Sigma}=
\begin{pmatrix}
\boldsymbol{\Sigma}_{OO} & \boldsymbol{\Sigma}_{OL}\\
\boldsymbol{\Sigma}_{LO} & \boldsymbol{\Sigma}_{LL}
\end{pmatrix}.
\]

Since \(\boldsymbol{\Sigma}\) is a stationary covariance matrix, it satisfies
\(\textbf{M}\boldsymbol{\Sigma}+\boldsymbol{\Sigma}\textbf{M}^\top+\textbf{C}=0\).
Taking the \((O,O)\)-block gives
\[
\textbf{M}_{OO}\boldsymbol{\Sigma}_{OO}
+\textbf{M}_{OL}\boldsymbol{\Sigma}_{LO}
+\boldsymbol{\Sigma}_{OO}\textbf{M}_{OO}^\top
+\boldsymbol{\Sigma}_{OL}\textbf{M}_{OL}^\top
+\textbf{C}_{OO}
=0.
\]

Because \(\boldsymbol{\Sigma}\) is positive definite, \(\boldsymbol{\Sigma}_{OO}\) is invertible. By the
definition
\(\boldsymbol{\Gamma}_{L\mid O}
=
\boldsymbol{\Sigma}_{LO}\boldsymbol{\Sigma}_{OO}^{-1}\),
we have
\(\boldsymbol{\Sigma}_{LO}
=
\boldsymbol{\Gamma}_{L\mid O}\boldsymbol{\Sigma}_{OO}\).
Moreover, using the symmetry of \(\boldsymbol{\Sigma}\),
\(\boldsymbol{\Sigma}_{OL}
=
\boldsymbol{\Sigma}_{LO}^\top
=
\boldsymbol{\Sigma}_{OO}\boldsymbol{\Gamma}_{L\mid O}^\top\).

Therefore,
\(\textbf{M}_{OL}\boldsymbol{\Sigma}_{LO}
=
\textbf{M}_{OL}\boldsymbol{\Gamma}_{L\mid O}\boldsymbol{\Sigma}_{OO}\),
and
\(\boldsymbol{\Sigma}_{OL}\textbf{M}_{OL}^\top
=
\boldsymbol{\Sigma}_{OO}\boldsymbol{\Gamma}_{L\mid O}^\top \textbf{M}_{OL}^\top\).

Substituting these expressions into the \((O,O)\)-block equation yields
\[
\bigl(\textbf{M}_{OO}+\textbf{M}_{OL}\boldsymbol{\Gamma}_{L\mid O}\bigr)\boldsymbol{\Sigma}_{OO}
+
\boldsymbol{\Sigma}_{OO}
\bigl(\textbf{M}_{OO}+\textbf{M}_{OL}\boldsymbol{\Gamma}_{L\mid O}\bigr)^\top
+
\textbf{C}_{OO}
=0.
\]
Defining
\(\textbf{M}^O:=\textbf{M}_{OO}+\textbf{M}_{OL}\boldsymbol{\Gamma}_{L\mid O}\)
gives
\(\textbf{M}^O\boldsymbol{\Sigma}_{OO}
+\boldsymbol{\Sigma}_{OO}(\textbf{M}^O)^\top
+\textbf{C}_{OO}
=0\),
as claimed.
\end{proof}

\subsection{Proof of Proposition 3}

\begin{proof}
Since the stationary covariance matrix satisfies the continuous Lyapunov equation, its \((y,z)\)-entry is
\((\textbf{M}\boldsymbol{\Sigma})_{yz}+(\boldsymbol{\Sigma}\textbf{M}^\top)_{yz}+c_{yz}=0\).

Because \(\textbf{C}\) is diagonal and \(z\neq y\), we have \(c_{yz}=0\). Moreover,
\((\textbf{M}\boldsymbol{\Sigma})_{yz}
=\sum_{v\in V}m_{yv}\Sigma_{vz}\).
Since \(m_{yv}=0\) whenever \(v\notin\operatorname{Pa}(y)\),
\((\textbf{M}\boldsymbol{\Sigma})_{yz}
=\sum_{p\in P}m_{yp}\Sigma_{pz}
+\sum_{u\in U}m_{yu}\Sigma_{uz}\),
where \(P=\operatorname{Pa}(y)\cap O\) and
\(U=\operatorname{Pa}(y)\cap L\).

Similarly,
\((\boldsymbol{\Sigma}\textbf{M}^\top)_{yz}
=\sum_{v\in V}\Sigma_{yv}(\textbf{M}^\top)_{vz}
=\sum_{v\in V}\Sigma_{yv}m_{zv}\).
Since \(m_{zv}=0\) whenever \(v\notin\operatorname{Pa}(z)\),
\((\boldsymbol{\Sigma}\textbf{M}^\top)_{yz}
=\sum_{q\in\operatorname{Pa}(z)}m_{zq}\Sigma_{yq}\).

Substituting these expressions into the \((y,z)\)-entry of the Lyapunov equation yields
\[
\sum_{p\in P}m_{yp}\Sigma_{pz}
+\sum_{u\in U}m_{yu}\Sigma_{uz}
+\sum_{q\in\operatorname{Pa}(z)}m_{zq}\Sigma_{yq}
=0,
\]
as claimed.
\end{proof}

\subsection{Proof of Lemma 3}

\begin{proof}
    Since $z$ has no parents other than itself, the parents set is $\operatorname{Pa}(z)=\{z\}$. Hence $b_{z} = \sum_{q\in\operatorname{Pa}(z)} m_{zq}\Sigma_{yq} = m_{zz}\Sigma_{yz}$.

    It remains to show that $m_{zz}$ is identifiable. Consider the $(z,z)$-entry of the Lyapunov equation $2\sum_{v\in \operatorname{Pa}(z)} m_{zv}\Sigma_{v z} + c_{zz} = 0$. Since $Pa(z)=\{z\}$, hence $2m_{zz}\Sigma_{zz} + c_{zz} = 0$. 

    As $\boldsymbol{\Sigma}$ is a covariance matrix and $z$ has nonzero stationary variance, we have: $\Sigma_{zz}>0$. Thus $m_{zz} = -\frac{c_{zz}}{2\Sigma_{zz}}$. Both $c_{zz}$ and $\Sigma_{zz}$ are observable or assumed known under the diagonal diffusion parametrization, and hence $b_{z} = m_{zz}\Sigma_{yz}$ and, consequently, $b_z$ is identifiable.
\end{proof}

\subsection{Proof of Theorem \ref{thm:directIV}}

\begin{proof}
We first derive the covariance relation responsible for the additional rank. The \((z,y)\)-entry of the Lyapunov equation is
\(\sum_{p \in P'}m_{yp}\Sigma_{pz} + m_{yy}\Sigma_{yz} + m_{zz}\Sigma_{yz} = 0\).
Hence
\begin{equation}\label{eq:cov_matching}
    \Sigma_{yz} = -\frac{\displaystyle\sum_{p\in P'}m_{yp}\Sigma_{pz}}{m_{yy}+m_{zz}}.
\end{equation}

\paragraph{Sufficiency}
Suppose there exist \(a\), \(z'\), \(z''\), and a matching as in the statement. Let \(P' \setminus \{a\} = \{q_1, \cdots, q_{r-1}\}\) and let \(w_i \in Z\setminus\{z', z''\}\) be the instrument matched to \(q_i\). Consider the \((r+1) \times (r+1)\) submatrix
\[
\textbf{K} = \boldsymbol{\Sigma}_{PZ} = \begin{pmatrix}
\boldsymbol{\Sigma}_{P',\,\{z',z'',w_1,\ldots,w_{r-1}\}}\\[1mm]
\boldsymbol{\Sigma}_{y,\,\{z',z'',w_1,\ldots,w_{r-1}\}}
\end{pmatrix}.
\]

To prove generic nonsingularity, it is enough to exhibit one parameter choice for which \(\det (\textbf{K})\neq0\). Indeed, after clearing the nonzero denominators in \eqref{eq:cov_matching}, \(\det (\textbf{K})\) is a rational function whose numerator is a polynomial in the model parameters. Showing that it is nonzero at one parameter point proves that this polynomial is not identically zero.

We construct a parameter specialization of the drift matrix. Retain only the off-diagonal coefficients corresponding to the edges \(z'\to a\), \(z''\to a\), \(w_i\to q_i\), for \(i=1,\ldots,r-1\), together with the edges \(p\to y\), for \(p\in P\), and set all remaining off-diagonal coefficients of \(\textbf{M}\) to zero. The parameters can be chosen with negative diagonal entries and sufficiently small off-diagonal entries, so stability is preserved.

Under these values for parameters, the covariance submatrix on the rows \(q_1,\ldots,q_{r-1}\) and columns \(w_1,\ldots,w_{r-1}\) is diagonal: \(\Sigma_{q_iw_j}=0\) for \(i\neq j\), and \(\Sigma_{q_iw_i}\neq0\). Likewise, \(\Sigma_{az'}\neq0\), \(\Sigma_{az''}\neq0\), while the remaining entries in the \(z'\)- and \(z''\)-columns vanish on the rows \(P'\setminus\{a\}\).

Consequently, after a permutation of rows and columns, the determinant of \(K\) factors as
\[
\lvert \det (\textbf{K}) \rvert=
\left(\prod_{i=1}^{r-1}\lvert \Sigma_{q_iw_i} \rvert\right)
\left| \det \begin{pmatrix}
\Sigma_{az'} & \Sigma_{az''}\\
\Sigma_{yz'} & \Sigma_{yz''}
\end{pmatrix} \right|.
\]

Under the chosen parametrization, in the equation \eqref{eq:cov_matching}, only the parent \(a\) contributes to the \(z'\)- and \(z''\)-columns, we obtain
\(\Sigma_{yz'}=-\frac{m_{ya}\Sigma_{az'}}{m_{yy}+m_{z'z'}}\) and
\(\Sigma_{yz''}=-\frac{m_{ya}\Sigma_{az''}}{m_{yy}+m_{z''z''}}\).

Therefore,
\[
\det
\begin{pmatrix}
\Sigma_{az'} & \Sigma_{az''}\\
\Sigma_{yz'} & \Sigma_{yz''}
\end{pmatrix}
=
m_{ya}\Sigma_{az'}\Sigma_{az''}
\frac{m_{z''z''}-m_{z'z'}}
     {(m_{yy}+m_{z'z'})(m_{yy}+m_{z''z''})}.
\]
This expression is nonzero because \(m_{ya}\neq0\), \(\Sigma_{az'}\neq0\), \(\Sigma_{az''}\neq0\), \(m_{z'z'}\neq m_{z''z''}\). It follows that \(\det (\textbf{K})\neq0\). Hence, generically,
\[
\operatorname{rank}(\boldsymbol{\Sigma}_{PZ}) =
\operatorname{rank}\begin{pmatrix}
\boldsymbol{\Sigma}_{P'Z}\\
\boldsymbol{\Sigma}_{yZ}
\end{pmatrix}
= r+1 = \lvert P \rvert .
\]

\paragraph{Necessity}
Suppose, generically,
\(\operatorname{rank} (\boldsymbol{\Sigma}_{PZ}) = \operatorname{rank}\begin{pmatrix}
\boldsymbol{\Sigma}_{P'Z}\\\boldsymbol{\Sigma}_{yZ}\end{pmatrix}
=r+1 = \lvert P \rvert\). Since adding one row can increase rank by at most one, necessarily \(\operatorname{rank}(\boldsymbol{\Sigma}_{P'Z})=r\).

Under the assumption that every \(Z\)-to-\(P'\) path is a single edge, the generic support pattern of \(\boldsymbol{\Sigma}_{P'Z}\) is precisely the bipartite graph \(B\). Therefore, the generic rank theorem for bipartite matrices implies that \(B\) contains a matching covering \(P'\). Moreover, an \((r+1)\times(r+1)\) nonsingular minor must use \(r+1\) distinct instrument columns. Choose such a set \(\widetilde Z\subseteq Z\), \(\lvert\widetilde Z\rvert=r+1\). Because \(\Sigma_{P',\widetilde Z}\) has rank \(r\), its support graph contains a matching \(\mathcal{H_M}\) covering \(P'\). This matching uses exactly \(r\) instruments, leaving one instrument \(z''\in\widetilde Z\) unmatched. The \(z''\)-column cannot vanish on all rows indexed by \(P'\), because then \eqref{eq:cov_matching} would also gives \(\Sigma_{yz''}=0\), and the corresponding column of the augmented matrix would be zero. Hence \(z''\) is adjacent to at least one parent, say \(a\in P'\). Let \(z'\) denote the instrument matched to \(a\) by \(\mathcal{H_M}\). Then \(z'\to a\), \(z''\to a\). After deleting \(a\), \(z'\), and \(z''\), the remaining edges of \(\textbf{M}\) constitute a matching covering \(P'\setminus\{a\}\).

Finally, the inequality \(m_{z'z'}\neq m_{z''z''}\) does not need to be imposed as an additional condition in a generic rank statement. From~\eqref{eq:cov_matching}, the rank augmentation associated with \(z'\) and \(z''\) contains the factor
\(-\frac{1}{m_{yy}+m_{z'z'}}+\frac{1}{m_{yy}+m_{z''z''}}\),
which vanishes precisely when \(m_{z'z'}=m_{z''z''}\). Since the diagonal drift coefficients are free parameters, this equality defines a proper algebraic subset of the parameter space and is excluded by genericity. Hence unequal self-drifts are not part of the graphical characterization. This completes the proof.

\end{proof}

\subsection{Proof of Theorem \ref{thm:nondirectIV}}

\begin{proof}
Let \(\vert P\rvert = \lvert P' \rvert +1 = r + 1\). It is enough to show one admissible parameter point at which an \((r+1) \times (r+1)\) minor of \(\boldsymbol{\Sigma}_{PZ}\) is nonzero. The determinant of every such minor is a rational function of the model parameters. If it is nonzero at one admissible parameter point, it is nonzero outside a proper algebraic subset of the parameter space.

Consider the columns indexed by \(Z_0=\{w_1,\ldots,w_{r-1},z',z''\}\) and order the rows as \(q_1,\ldots,q_{r-1},a,y\). Let \(\textbf{K}:=\boldsymbol{\Sigma}_{\{q_1,\ldots,q_{r-1},a,y\},Z_0}\).

We construct an instantiation of parameters. In the drift matrix, retain only the off-diagonal coefficients corresponding to the selected paths \(\pi'\), \(\pi''\), and \(\rho_1,\ldots,\rho_{r-1}\), together with the parent edges \(p\to y\), \(p\in P\), and set all other off-diagonal coefficients to zero. Choose the path coefficients to be nonzero.

By the vertex-disjointness assumptions and the instrumental-validity condition, the internal vertices of the selected paths cannot belong to \(P\cup\{y\}\). Consequently, the additional edges \(p\to y\) cannot create a directed connection between two distinct selected path systems.

Finally, choose the diagonal drift coefficients to be sufficiently
negative. More precisely, after fixing the off-diagonal coefficients,
choose \(m_{vv}<0\) with sufficiently large magnitude so that every
Gershgorin disc of \(\textbf{M}\) lies strictly in the open left
half-plane. The resulting drift matrix is therefore stable. We may,
in addition, choose \(m_{z'z'}\neq m_{z''z''}\).

Under this instantiation, the covariance submatrix on the rows \(q_1,\ldots,q_{r-1}\) and columns \(w_1,\ldots,w_{r-1}\) is diagonal: \(\Sigma_{q_iw_j}=0\) for \(i\neq j\), and \(\Sigma_{q_iw_i}\neq0\) whereas the directed path \(\rho_i:w_i\rightsquigarrow q_i\). Consequently, after the chosen ordering, \(\boldsymbol{\Sigma}_{\{q_1,\ldots,q_{r-1}\},\{w_1,\ldots,w_{r-1}\}}\) is diagonal with nonzero diagonal entries. Hence, the determinant is nonzero.

Similarly, because every \(\rho_i\) is vertex-disjoint from \(\pi'\cup\pi''\), there is no trek between \(q_i\) and either \(z'\) or \(z''\) under this instantiation. Therefore, \(\Sigma_{q_i z'}=\Sigma_{q_i z''}=0\), for \(i=1,\ldots,r-1\). Thus \(\textbf{K}\) has the block lower-triangular form
\[
\textbf{K}=
\begin{pmatrix}
\boldsymbol{\Sigma}_{\{q_1,\ldots,q_{r-1}\},\{w_1,\ldots,w_{r-1}\}} & 0\\
\textbf{B} & \boldsymbol{\Sigma}_{\{a,y\},\{z',z''\}}
\end{pmatrix},
\]
where the block \(\textbf{B}\) need not vanish.

In particular,
\(\det(\textbf{K})
=
\det(\boldsymbol{\Sigma}_{\{q_1,\ldots,q_{r-1}\},\{w_1,\ldots,w_{r-1}\}})
\det(\boldsymbol{\Sigma}_{\{a,y\},\{z',z''\}})\).

It remains to show that \(\det(\boldsymbol{\Sigma}_{\{a,y\},\{z',z''\}})\neq0\). The directed paths \(\pi'\) and \(\pi''\) imply, generically, \(\Sigma_{az'}\neq0\) and \(\Sigma_{az''}\neq0\).

Because \(z'\) is a root and has no direct edge to \(y\), the \((y,z')\)-entry of the Lyapunov equation is
\(\sum_{p\in P'}m_{yp}\Sigma_{pz'} + (m_{yy}+m_{z'z'})\Sigma_{yz'} = 0\).
Under the selected instantiation, \(\Sigma_{pz'}=0\) for every \(p\in P'\setminus\{a\}\). Hence
\((m_{yy}+m_{z'z'})\Sigma_{yz'} + m_{ya}\Sigma_{az'} = 0\),
and therefore
\(\Sigma_{yz'} = -\frac{m_{ya}\Sigma_{az'}}{m_{yy}+m_{z'z'}}\) and
\(\Sigma_{yz''} = -\frac{m_{ya}\Sigma_{az''}}{m_{yy}+m_{z''z''}}\).

The denominators are nonzero because all diagonal drift coefficients were chosen strictly negative in this specialization. Therefore,
\[
\det(\boldsymbol{\Sigma}_{\{a,y\},\{z',z''\}})
=
\det
\begin{pmatrix}
\Sigma_{az'} & \Sigma_{az''}\\
\Sigma_{yz'} & \Sigma_{yz''}
\end{pmatrix}
=
m_{ya}\Sigma_{az'}\Sigma_{az''}
\frac{m_{z''z''}-m_{z'z'}}
     {(m_{yy}+m_{z'z'})(m_{yy}+m_{z''z''})}.
\]

This quantity is nonzero because \(m_{ya}\neq0\), \(\Sigma_{az'}\neq0\), \(\Sigma_{az''}\neq0\), and \(m_{z'z'}\neq m_{z''z''}\).

It follows that \(\det(\textbf{K})\neq0\). Hence \(\boldsymbol{\Sigma}_{PZ}\) contains a nonzero \((r+1)\times(r+1)\) minor, and therefore, generically, \(\operatorname{rank}(\boldsymbol{\Sigma}_{PZ})=r+1=|P|\).

\end{proof}

\section{Further Details on the Four-Node Extensions}

This section provides additional details on the four-node extensions associated with Proposition~\ref{prop:one-vertex-extensions}. In particular, we illustrate the identifiable configurations covered by the proposition and give examples of extensions that fail to restore generic identifiability.

The examples below emphasize that adding a fourth vertex is not sufficient by itself to repair the rank deficiency of a deficient three-node motif. What matters is whether the added structure changes the trek and covariance structure in a way that introduces an additional independent constraint in
$\mathbf{H}(\boldsymbol{\Sigma})_{E^c,\cdot}$.

\subsection{Identifiable Four-Node Extensions}

Figure~\ref{fig:four_node_motifs} shows representative four-node configurations that are generically identifiable. The red arrows indicate admissible extensions involving the added vertex $v$. A solid red arrow is required, whereas a dashed red arrow denotes an optional edge. Blue connections indicate adjacencies whose orientation may be arbitrary; whenever several blue adjacencies are shown, at least one of the indicated connections must be present.

\begin{figure}[h]
    \centering
    \begin{subfigure}[t]{0.21\columnwidth}
        \centering
        \includegraphics[page=1,width=\linewidth]{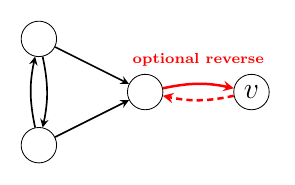}
        \caption{Common child.}
        \label{fig:motif_a}
    \end{subfigure}
    \hfill
    \begin{subfigure}[t]{0.21\columnwidth}
        \centering
        \includegraphics[page=4,width=\linewidth]{Figures/motifs_4v.pdf}
        \caption{Common parent.}
        \label{fig:motif_b}
    \end{subfigure}
    \hfill
    \begin{subfigure}[t]{0.21\columnwidth}
        \centering
        \includegraphics[page=13,width=\linewidth]{Figures/motifs_4v.pdf}
        \caption{$2$-cycles intersection.}
        \label{fig:motif_c}
    \end{subfigure}
    \hfill
    \begin{subfigure}[t]{0.21\columnwidth}
        \centering
        \includegraphics[page=10,width=\linewidth]{Figures/motifs_4v.pdf}
        \caption{$2$-cycles intersection.}
        \label{fig:motif_d}
    \end{subfigure}
    
    \begin{subfigure}[t]{0.12\columnwidth}
        \centering
        \includegraphics[page=8,width=\linewidth]{Figures/motifs_4v.pdf}
        \caption{Coupling two $2$-cycles.}
        \label{fig:motif_e}
    \end{subfigure}
    \begin{subfigure}[t]{0.18\columnwidth}
        \centering
        \includegraphics[page=24,width=\linewidth]{Figures/AAAI_2027-38.pdf}
        \caption{Asymmetric extension.}
        \label{fig:motif_f}
    \end{subfigure}

    \caption{Four-node motifs that are generically identifiable. Red arrows indicate admissible extensions involving the added vertex $v$. A solid red arrow is required, while a dashed red arrow denotes an optional choice. Blue lines indicate adjacency whose orientation is arbitrary. For a set of blue adjacencies, at least one of the indicated blue connections must be present, in either direction.}
    \label{fig:four_node_motifs}
\end{figure}

The extensions in Figures~\ref{fig:motif_a}--\ref{fig:motif_d} illustrate the configurations summarized in Proposition~\ref{prop:one-vertex-extensions}. These examples show that identifiability is not determined by adjacency alone. Rather, the added vertex must alter the covariance structure so that the missing rank contribution of the deficient three-node motif is recovered.

Figure~\ref{fig:motif_e} illustrates an additional rank-repair mechanism. Here, the graph consists of two vertex-disjoint directed $2$-cycles coupled by directed edges between their vertices. Although the two directed $2$-cycles are deficient when considered separately, their coupling as shown in Figure~\ref{fig:motif_e} renders the combined four-node graph generically identifiable.

Moreover, one of the connecting edges may be removed without destroying identifiability. Removing such an edge adds the corresponding non-edge row to
$\mathbf{H}(\boldsymbol{\Sigma})_{E^c,\cdot}$, while the remaining connection preserves the treks responsible for the covariance entries needed to obtain full column rank.

\subsection{Embedding the Motifs in Larger Graphs}

The identifiable configurations above need not occur as isolated four-node graphs. The same block-rank argument can be used when several such motifs are embedded in a larger graph.

Suppose that the graph contains a collection of vertex-disjoint four-node motifs of the types described above, no vertex belongs to more than one such motif, and the subgraph induced by the remaining vertices is generically identifiable. Assume additionally that no pair of vertices belonging to two distinct blocks---either two different four-node motifs, or a four-node motif and the remaining subgraph---forms a directed $2$-cycle.

Then the same block-rank construction used in the proof of Theorem~\ref{thm:3v_motif} extends to the full graph. The motif-local blocks have full column rank by the results above, the remaining block has full column rank by assumption, and the columns involving vertices from distinct blocks can be handled by selecting suitable non-edge rows as in the mixed-block argument of Theorem~\ref{thm:3v_motif}. Hence, the complete graph is generically identifiable.

\subsection{Failed Extensions}

The previous examples also show that introducing an additional vertex does not necessarily repair the rank deficiency. Figure~\ref{fig:non_four_node_motifs} shows representative four-node extensions that remain non-identifiable.

\begin{figure}[h]
    \centering
    \begin{subfigure}[t]{0.2\columnwidth}
        \centering
        \includegraphics[page=17,width=\linewidth]{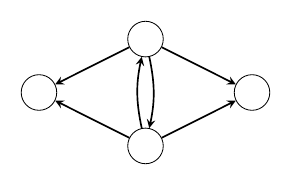}
        \caption{Duplicate common child.}
        \label{fig:non_motif_a}
    \end{subfigure}
    \begin{subfigure}[t]{0.2\columnwidth}
        \centering
        \includegraphics[page=18,width=\linewidth]{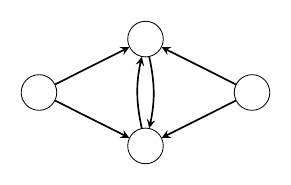}
        \caption{Duplicate common parent.}
        \label{fig:non_motif_b}
    \end{subfigure}
    \begin{subfigure}[t]{0.2\columnwidth}
        \centering
        \includegraphics[page=19,width=\linewidth]{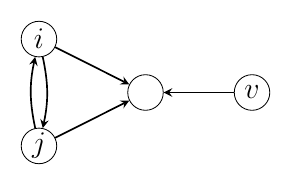}
        \caption{Common-child failed extension.}
        \label{fig:non_motif_c}
    \end{subfigure}
    \caption{Four-node extensions that remain non-identifiable.}
    \label{fig:non_four_node_motifs}
\end{figure}

The condition that the added structure creates a relevant trek connection with the vertices in the directed $2$-cycle is essential. Consider, for example, the common-child motif and the extension shown in Figure~\ref{fig:non_motif_c}. In this configuration, there is no trek from the added vertex $v$ to either vertex of the directed $2$-cycle $(i,j)$. By the trek rule,
$\Sigma_{vi}=\Sigma_{vj}=0$. Consequently, the additional rows involving $v$ do not introduce new covariance entries capable of repairing the missing rank contribution. The system becomes larger, but the original rank deficiency persists.

Extensions that merely duplicate the common-child or common-parent structures also fail to restore identifiability. As shown in Figures~\ref{fig:non_motif_a} and~\ref{fig:non_motif_b}, these symmetric attachments reproduce constraints already present in
$\mathbf{H}(\boldsymbol{\Sigma})_{E^c,\cdot}$ and therefore do not supply an additional independent rank contribution.

In contrast, the asymmetric extension shown in Figure~\ref{fig:motif_f} breaks this redundancy. The asymmetry changes the trek structure and introduces an additional independent covariance constraint, thereby restoring generic identifiability.

Overall, these examples reinforce the interpretation of Proposition~\ref{prop:one-vertex-extensions}: the decisive feature of an extension is not the mere presence of an additional node or edge, but whether the resulting trek structure generates new covariance information that increases the rank of $\mathbf{H}(\boldsymbol{\Sigma})_{E^c,\cdot}$.

\section{Trek graphs and Proposition~2}

This section provides the graphical interpretation of
Proposition~2 in the main paper. In particular, we explain how
conditional independence in a stationary diffusion can be read from the trek graph and why the condition $X_h \perp X_x \mid X_{O\setminus\{x\}}$ eliminates the marginalization term associated with a latent parent $h$ of $y$.

\paragraph{Treks and the trek graph.}

\begin{definition}[Trek graph {\cite{boege2025conditional}}]
\label{def:trek-graph}
Let \(G=(V,E)\) be the directed drift graph. Its \emph{trek graph}
\(\widehat G=(V,\widehat E)\) is the bidirected graph satisfying
\[
x\leftrightarrow y\in\widehat E
\quad\Longleftrightarrow\quad
\operatorname{An}_G(x)\cap\operatorname{An}_G(y)\neq\emptyset,
\qquad x\neq y.
\]
Equivalently, \(x\leftrightarrow y\in\widehat E\) if and only if
there exists a trek between \(x\) and \(y\) in \(G\).
\end{definition}

The equivalence in Definition~\ref{def:trek-graph} follows directly from the definition of a trek. If a trek between \(x\) and \(y\) has top \(c\), then \(c\in\operatorname{An}_G(x)\cap\operatorname{An}_G(y)\). Conversely, any
\(c\in\operatorname{An}_G(x)\cap\operatorname{An}_G(y)\) gives two directed paths \(c\rightsquigarrow x\) and \(c\rightsquigarrow y\), which together form a trek between \(x\) and \(y\).

\begin{definition}[Separation in the trek graph]
\label{def:trek-separation}
Let \(I,J,S\subseteq V\) be pairwise disjoint. We write
\(I\perp_{\widehat G}J\mid S\)
if every path in \(\widehat G\) from a vertex in \(I\) to a vertex
in \(J\) intersects \(S\).
\end{definition}

Since the linear Ornstein--Uhlenbeck model considered in this paper is a special case of the stationary diffusions covered by Theorem~1.2 \cite{boege2025conditional}, the separation condition in its first statement implies the corresponding conditional independence in our model.

\begin{corollary}[Graphical condition for Proposition~2]
\label{cor:prop2-separation}
Let \(V=O\mathbin{\dot\cup}L\), let \(h\in L\), and let \(x\in O\).
If \(\{h\}\perp_{\widehat G}\{x\}\mid L\setminus\{h\}\), then
\(X_h\perp X_x\mid X_{O\setminus\{x\}}\).
\end{corollary}

\begin{proof}
Apply Theorem~1.2 \cite{boege2025conditional} with \(I=\{h\}\), \(J=\{x\}\), and \(K=O\setminus\{x\}\). Because \(V=O\mathbin{\dot\cup}L\),
\[
\begin{aligned}
V\setminus(I\cup J\cup K)
&=
V\setminus
\left(
\{h\}\cup\{x\}\cup(O\setminus\{x\})
\right)\\
&=
V\setminus(\{h\}\cup O)\\
&=
L\setminus\{h\}.
\end{aligned}
\]
The claim therefore follows from Theorem~1.2 \cite{boege2025conditional}.
\end{proof}

\begin{lemma}
\label{lem:regression-coefficient}
Let \(S=O\setminus\{x\}\). Then
\[
(\Gamma_{L\mid O})_{hx}
=
\frac{
\Sigma_{hx}
-
\Sigma_{hS}\Sigma_{SS}^{-1}\Sigma_{Sx}
}{
\Sigma_{xx}
-
\Sigma_{xS}\Sigma_{SS}^{-1}\Sigma_{Sx}
}.
\]
Consequently,
\(X_h\perp X_x\mid X_S
\Longrightarrow
(\Gamma_{L\mid O})_{hx}=0\).
\end{lemma}

\begin{proof}
Since \(\Gamma_{L\mid O}=\Sigma_{LO}\Sigma_{OO}^{-1}\), the expression follows from the block inverse formula applied to the partition \(O=\{x\}\mathbin{\dot\cup}S\). The numerator is the conditional covariance \(\Sigma_{hx\mid S}=\Sigma_{hx}-\Sigma_{hS}\Sigma_{SS}^{-1}\Sigma_{Sx}\), and the denominator is the conditional variance
\(\Sigma_{xx\mid S}>0\). Since the stationary distribution is Gaussian,
\(X_h\perp X_x\mid X_S
\Longleftrightarrow
\Sigma_{hx\mid S}=0\).
\end{proof}

So, suppose \(h\to y\) and \(x\to y\), where \(h\in L\) and \(x,y\in O\).
Then \(\widehat G\) contains \(h-y-x\), because each directed edge is a trek. The required separator is
\(L\setminus\{h\}\), but the internal vertex \(y\) is observed. Therefore,
\(\{h\}\not\perp_{\widehat G}\{x\}\mid L\setminus\{h\}\),
so Theorem~1.2 \cite{boege2025conditional} does not guarantee
\(X_h\perp X_x\mid X_{O\setminus\{x\}}\).
This motivates the instrumental-variable approach.

\section{Computational details}

All experiments were implemented in Python 3.13 using NumPy, SciPy, NetworkX, and Matplotlib. Computations were performed on a MacBook Air equipped with an Apple M4 processor and 16 GB of unified memory.

The complete source code used to generate all synthetic experiments, together with the parameter initializations and implementation details, is included in the supplementary material.

No trajectories were simulated in any experiment. All evaluations were performed at the population level. For every generated model, the stationary covariance matrix was obtained directly by solving the continuous Lyapunov equation
\[
\textbf{M}\boldsymbol{\Sigma}+\boldsymbol{\Sigma} \textbf{M}^\top+\textbf{C}=0.
\]

\subsection{Generic Identifiability}

The numerical experiments investigate generic identifiability of individual entries of the drift matrix for Theorem 1. The same procedure was used for the graphical configurations displayed in Figure~2.

\subsubsection{Graph Generation}

We considered directed graph structures containing one or more directed $2$-cycles. Depending on the experiment, the graphs were either fixed to represent the local configurations shown in Figure~2 or generated according to the graphical assumptions under investigation. For each directed $2$-cycle, additional adjacent vertices and edges were introduced according to the corresponding graph configuration. All remaining entries of the drift graph were set to zero unless explicitly included in the construction.

\subsubsection{Model Specification}

For each experiment, we fixed a directed graph and constructed the
corresponding symbolic drift matrix \(\textbf{M}\). An off-diagonal entry \(m_{vu}\) was treated as a free parameter precisely when the graph contained the edge \(u\to v\). The diagonal entries of \(\textbf{M}\) were also treated as free parameters. The diffusion matrix \(\textbf{C}\) was restricted to be diagonal with strictly positive entries.

Where required by the graphical configuration, selected covariance entries were constrained to zero before solving the Lyapunov equations.

\subsubsection{Parameter Sampling}

For each trial, positive parameters were sampled independently from
\((10^{-10},10]\), negative parameters from \([-10,-10^{-10})\), and unconstrained off-diagonal drift parameters from \([-10,10]\). A sampled drift matrix was retained only when all of its eigenvalues had strictly negative real parts. Up to \(100\) samples were drawn to obtain a Hurwitz drift matrix for each trial.

\subsubsection{Covariance Computation}

For every accepted drift matrix, the stationary covariance entries were obtained by solving the linear system induced the continuous Lyapunov equation. Samples for which the system was inconsistent or the resulting covariance matrix was not positive definite were discarded.

\subsubsection{Numerical Identifiability Test}

After fixing the generated stationary covariance matrix and the diagonal diffusion parameters, identifiability was tested separately for each free entry of the drift matrix. We minimized and maximized the selected parameter subject to the Lyapunov equations, while allowing the remaining drift parameters to vary within \([-10,10]\). Both optimization problems were solved using linear programming.

Let \(\underline m_{vu}\) and \(\overline m_{vu}\) denote the resulting minimum and maximum feasible values. The parameter \(m_{vu}\) was classified as numerically identifiable whenever
\[
\left|\overline m_{vu}-\underline m_{vu}\right|<10^{-8}.
\]
For each graphical configuration and each drift parameter, this procedure was repeated over \(1000\) random trials. We report the proportion of valid trials for which the parameter was classified as identifiable.

\subsection{Instrumental Variables}

\begin{figure}[t]
    \centering

    \begin{minipage}[t]{0.3\columnwidth}
        \centering
        \includegraphics[width=\linewidth]{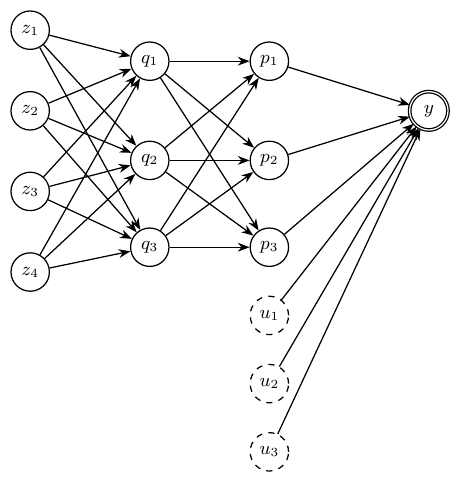}

        \small (a) Dense configuration.
    \end{minipage}
    \begin{minipage}[t]{0.3\columnwidth}
        \centering
        \includegraphics[width=\linewidth]{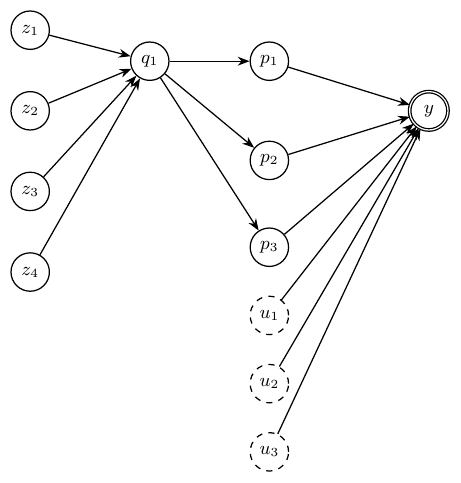}

        \small (b) Bottleneck configuration.
    \end{minipage}

    \caption{IV setup}
    \label{fig:iv_setup}
\end{figure}

\subsubsection{Graph Configurations}

We considered fixed graph configurations containing a target vertex \(y\), observed parents \(P\), latent parents \(U\), and root instruments \(Z\). 
We considered 
configurations violating the bottleneck (reported in the main paper). The graph structure was fixed within each configuration, while the numerical model parameters were independently resampled.

Figure~\ref{fig:iv_setup} illustrates two representative instrumental-variable configurations. The left panel shows a sufficient configuration for the full-rank condition of Theorem~5, whereas the right panel illustrates a bottleneck configuration with bottleneck ratio $r=1/3$.

\subsubsection{Parameter Initialization}

For every repetition, the diagonal entries of the drift matrix were sampled independently from $[-10,-10^{-8}]$, and the coefficient associated with each directed edge was sampled independently from
$[-10,10]$. All remaining off-diagonal entries were set to zero. Samples were repeatedly generated until the resulting drift matrix was Hurwitz. The diffusion matrix was diagonal, with entries sampled independently from $[10^{-8},10]$.

The stationary covariance matrix was computed directly by solving
the continuous Lyapunov equation. No trajectories or finite observational samples were generated.

\subsubsection{Recovery Procedure}

Only the observed covariance matrix \(\boldsymbol{\Sigma}_{OO}\) was supplied to the recovery procedure. Since the instruments are roots, their diagonal drift coefficients were obtained from $m_{zz}=-\frac{c_{zz}}{2\Sigma_{zz}}$, and the corresponding vector \(b\) was constructed from observed covariance entries. The target coefficients and the self-drift of \(y\) were then estimated jointly by solving
\[
\begin{bmatrix}
\boldsymbol{\Sigma}_{P'Z}^{\top} & \boldsymbol{\Sigma}_{yZ}^{\top}
\end{bmatrix}
\begin{bmatrix}
\widehat m_{yP'}\\
\widehat m_{yy}
\end{bmatrix}
=-b
\]
using least squares.

\subsubsection{Evaluation Metric}

Recovery accuracy was evaluated using the relative $\ell_2$ error between the estimated and true vectors of target drift coefficients,
$\frac{\|\hat m_{yP'}-m_{yP'}\|_2}{\|m_{yP'}\|_2}$, where \(m_{yP'}\) denotes the true vector of target drift coefficients and \(\hat m_{yP'}\) its estimate. The reported results correspond to the mean relative error across 100 independent repetitions.

\subsection{Restrictiveness of Theorem~1}

To assess how restrictive the sufficient graphical condition of Theorem~1 is relative to the kernel-rank criterion of Lemma~1, we generated random directed graphs with
\(|V|\in\{6,8,10,12\}\). Each possible off-diagonal directed edge was included independently with probability
\(p\in\{0.05,0.10,0.15,0.20,0.25,0.30\}\).
For each combination of \(|V|\) and \(p\), we retained \(500\) graphs containing at least one directed \(2\)-cycle; graphs without directed \(2\)-cycles were excluded. For each retained graph, we checked whether the graphical conditions of Theorem~1 were satisfied.  

We then numerically evaluated the generic kernel-rank condition of Lemma~1. For each directed edge \(u\to v\), the corresponding drift coefficient \(m_{vu}\) was sampled independently from a standard normal distribution. The diagonal of the resulting drift matrix was shifted uniformly so that its spectral abscissa was at most \(-1\), thereby ensuring that the drift matrix was Hurwitz. The diffusion matrix was diagonal, with its diagonal entries sampled independently from \([0.5,2]\). The stationary covariance matrix \(\boldsymbol{\Sigma}\) was then obtained by solving the continuous Lyapunov equation.

Given \(\boldsymbol{\Sigma}\), we constructed
\(\textbf{H}(\boldsymbol{\Sigma})_{E^c,\cdot}\), whose rows correspond to directed non-edges \(i\to j\notin E\) and whose columns correspond to unordered vertex pairs \((k,l)\), \(k<l\). Its entries were computed as
\[
H(\boldsymbol{\Sigma})_{i\to j,(k,l)}
=
\begin{cases}
-\Sigma_{lj}, & i=k,\\
\Sigma_{kj}, & i=l,\\
0, & \text{otherwise}.
\end{cases}
\]
Graphs with fewer than \(|V|(|V|-1)/2\) directed non-edges were immediately classified as not having full kernel rank.

For the remaining graphs, the rank of
\(\textbf{H}(\boldsymbol{\Sigma})_{E^c,\cdot}\) was computed from its singular values. A singular value was treated as numerically nonzero when it exceeded
\[
\tau
=
100\,\max\{\operatorname{rows}(\textbf{H}),\operatorname{cols}(\textbf{H})\}
\,\epsilon_{\mathrm{mach}}\,s_{\max},
\]
where \(\epsilon_{\mathrm{mach}}\) is machine precision and \(s_{\max}\) is the largest singular value. Full kernel rank was declared when the resulting numerical rank equaled \(|V|(|V|-1)/2\). To reduce sensitivity to exceptional parameter choices, up to three independent parameter draws were tested for each graph, and the graph was classified as numerically full rank if the required rank was attained for at least one draw.

Finally, for each \((|V|,p)\) combination, we report the coverage of Theorem~1 among numerically full-rank graphs, defined as the fraction of graphs satisfying the graphical condition of Theorem~1 among those attaining numerical full kernel rank. The random number generator was initialized with seed \(42\).

\begin{figure}[t]
    \centering
    \includegraphics[width=0.6\columnwidth]
    {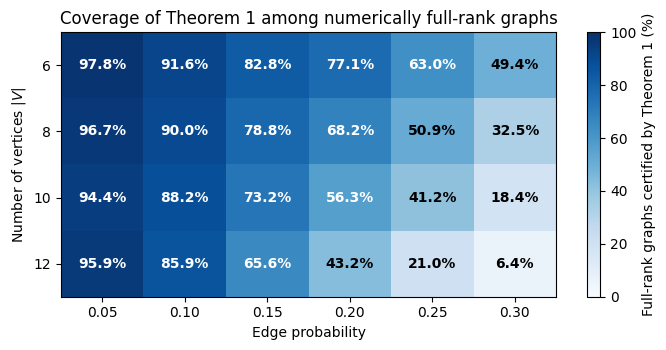}
    \caption{Coverage of the sufficient graphical condition in
    Theorem~1. Each cell reports the percentage of randomly generated
    non-simple graphs with numerical full kernel rank that are certified by Theorem~1, for different numbers of vertices $|V|$ and edge probabilities. 
    .}
    
    \label{fig:theorem1_coverage}
\end{figure}

Figure~\ref{fig:theorem1_coverage} shows that Theorem~1 captures most numerically full-rank graphs in the sparse regime. At \(p=0.05\), coverage ranges from \(94.4\%\) to \(97.8\%\) across the considered graph sizes. As the graphs become denser, the sufficient condition becomes increasingly conservative. For example, for \(|V|=12\), coverage decreases from \(95.9\%\) at \(p=0.05\) to \(6.4\%\) at \(p=0.30\). This behavior is consistent with the structural requirements of Theorem~1: denser graphs contain more overlapping directed \(2\)-cycles, which fall outside its vertex-disjoint construction even when the kernel matrix has full column rank.

\end{document}